\documentclass[letterpaper]{article} %
\usepackage{aaai25}  %
\usepackage{times}  %
\usepackage{helvet}  %
\usepackage{courier}  %
\usepackage[hyphens]{url}  %
\usepackage{graphicx} %
\usepackage{natbib}  %
\usepackage{caption} %
\usepackage{algorithm}
\usepackage{algpseudocode}
\usepackage{cite}
\usepackage{amsmath}
\usepackage{xcolor}
\usepackage{booktabs}
\usepackage{subcaption}
\usepackage{amsthm}

\allowdisplaybreaks[1]

\newtheorem{lemma}{Lemma}
\newtheorem*{lemma*}{Lemma}
\newtheorem{theorem}{Theorem}
\newtheorem*{theorem*}{Theorem}

\newtheorem*{condition*}{Condition}
\newtheorem{corollary}{Corollary}
\newtheorem*{corollary*}{Corollary}
\DeclareMathOperator*{\argmax}{arg\,max}
\DeclareMathOperator*{\argmin}{arg\,min}

\usepackage{newfloat}
\usepackage{listings}
\DeclareCaptionStyle{ruled}{labelfont=normalfont,labelsep=colon,strut=off} %
\floatstyle{ruled}
\newfloat{listing}{tb}{lst}{}
\floatname{listing}{Listing}
\nocopyright %

\title{Tracing Privacy Leakage of Language Models to Training Data via Adjusted Influence Functions}
\author{
    Jinxin Liu\textsuperscript{\rm 1} Zao Yang\textsuperscript{\rm 1}\\
}
\affiliations{
    \textsuperscript{\rm 1}Huawei Technologies Canada Co. Ltd\\
    300 Hagey Blvd, Waterloo, ON N2L 0A4, Canada\\
    \{jinxin.liu, zao.yang\}@huawei.com
}

\begin{document}

\maketitle

\begin{abstract}
The responses generated by Large Language Models (LLMs) can include sensitive information from individuals and organizations, leading to potential privacy leakage. This work implements Influence Functions (IFs) to trace privacy leakage back to the training data, thereby mitigating privacy concerns of Language Models (LMs). 
However, we notice that current IFs struggle to accurately estimate the influence of tokens with large gradient norms, potentially overestimating their influence. When tracing the most influential samples, this leads to frequently tracing back to samples with large gradient norm tokens, overshadowing the actual most influential samples even if their influences are well estimated.
To address this issue, we propose Heuristically Adjusted IF (HAIF), which reduces the weight of tokens with large gradient norms, thereby significantly improving the accuracy of tracing the most influential samples.
To establish easily obtained groundtruth for tracing privacy leakage, we construct two datasets, PII-E and PII-CR, representing two distinct scenarios: one with identical text in the model outputs and pre-training data, and the other where models leverage their reasoning abilities to generate text divergent from pre-training data.
HAIF significantly improves tracing accuracy, enhancing it by 20.96\% to 73.71\% on the PII-E dataset and 3.21\% to 45.93\% on the PII-CR dataset, compared to the best SOTA IFs against various GPT-2 and QWen-1.5 models.
HAIF also outperforms SOTA IFs on real-world pretraining data CLUECorpus2020, demonstrating strong robustness regardless prompt and response lengths.
\end{abstract}

\section{Introduction}
\label{sec:intro}
With the emergence of ChatGPT, LLMs have received phenomenal social attentions due to their powerful capabilities. However, the safety and privacy concerns of LLMs continue to be evaluated and challenged by academia and industry \cite{zhou_comprehensive_2023}. LLMs extensively utilize data from public internet, private domain and user interaction during training stage. LLM pretraining data inevitably includes sensitive information from individuals and organizations which can be leaked during inference stage \cite{yao_survey_2024, sun_trustllm_2024, wang_decodingtrust_2024, yang_harnessing_2024}.
In the quest to protect privacy, LLM manufacturers have explored various strategies, including privacy cleaning of pre-training data, LLM alignment, and content moderation. However, due to jailbreak \cite{wei_jailbroken_2023,li_multi-step_2023} and training data extraction techniques \cite{nasr_scalable_2023}, these measures are not foolproof.

Consequently, when our privacy is exposed to LLM users despite multiple layers of protection, it raises an important question: Which training data is the root-cause of privacy leakage in a LLM? The answer to this question pertains to several aspects: the identification and tracing of privacy infringements, preventing the flow of sensitive information into further training models, and bridging down-stream tasks such as machine unlearning and model editing \cite{liu_rethinking_2024,chen_robust_2024}.

Since IFs are introduced into deep learning in 2017 \cite{koh_understanding_2017}, their primary objective is explaining the predictions of black-box models. 
IFs have been widely used for identifying the most influential training samples in various fields such as image classification \cite{fisher_influence_2023, lee_learning_2020}, text classification \cite{zylberajch_hildif_2021, schioppa_theoretical_2023} and language modeling \cite{grosse_studying_2023}. Therefore, this work implements IFs to tracing privacy leakage in LMs; however, while IFs seem promising for tracing privacy contents, there remains theoretical and practical open problems to be solved.

Effectively utilizing IFs requires assumptions that are often invalid in deep learning contexts, leading to significant errors, such as the convexity of empirical loss, the neglect of training trajectories, and the limitations of parameter divergence when up/down-weighting samples \cite{schioppa_theoretical_2023}. 
While other assumptions can be mitigated, Schioppa et al. question the assumption regarding parameter divergence limitations. By deriving an upper bound for parameter divergence when up/down-weighting samples, the upper bound can grow exponentially over time. This leads to the model parameters after Leave-One-Out-Retraining (LOOR) possibly no longer being near the original model parameters, causing the most basic assumption of IFs to fail, making them only effective within a very limited time steps. This implies that IFs cannot be applied to most deep learning models.

Furthermore, IFs have been found that regardless of which test sample is traced, the most influential samples are always traced back to training samples with large norm of gradient \cite{barshan_relatif_2020}. To avoid this issue, RelatIF is introduced with an additional constraint to limit the influence of large norm of gradient, leading to sample-wise normalization on original IFs. However, the following questions have not been well answered: Why do IFs need additional constraints? Which problem is this additional constraint fundamentally solving? Meanwhile, we observe that RelatIF cannot provide satisfactory results for tracing privacy leakage in our experiments.

To solve the aforementioned issues, the contributions of this paper are four-fold:
\begin{itemize}
    \item To the best of our knowledge, this work is the first to use IFs to trace privacy leakage in LMs, extending the application of IFs and further safeguarding the privacy of LLMs.
    \item This work reveals that the gradient norms of tokens affect the effectiveness of IFs in most deep learning models in three key aspects: the existence of IFs, the accuracy of LOOR parameter estimation, and the estimation errors of existing IFs.
    \item We propose HAIF which reduces the weights of tokens with large gradient norms, significantly improving the performance of identifying the most influential training samples with lower computational costs compared to the best SOTA IFs.
    \item Comprehensive experiments are conducted to demonstrate that the proposed HAIF significantly enhances tracing performance across various model types and parameter scales compared to SOTA IFs.
\end{itemize}

\section{Related Work and Preliminaries}
\label{sec:related work}
Typically, one can perform LOOR and observe the change of model parameters and prediction loss to check if a training sample is important for model parameters and predictions. Nonetheless, performing LOOR for each training sample is intractable given the current scale of LM parameters and dataset. Thus, IFs are proposed to simulate such process aiming at finding the most influential training samples of model parameters and predictions \cite{koh_understanding_2017}.

Let $D=\{z_1,...,z_n\}$ be a training set comprising $n$ samples, $\theta$ represent the model parameters and $L$ denote a loss function for each training sample. Assume $\theta_{0}$ as the original trained model parameters and $\theta_{\epsilon_k}$ as the model trained after the up/down weighting $z_k \in D$ by $\epsilon_k$. The influence of $z_k$ on the parameters can be expressed as:
\begin{equation}
    \label{eq:if param def}
    I_{param}(z_k) := \frac{d(\theta_{\epsilon_k} - \theta_{0})}{d \epsilon_k} \bigg|_{\epsilon_k=0},
\end{equation}
The influence of $z_k$ on the loss of a test sample $z_{test}$ is quantified by:
\begin{equation}
    I_{loss}(z_{test}, z_k) := \nabla_\theta L(z_k, \theta_0) I_{param}(z_k).
\end{equation}

IFs can be categorized into two broad types based on the method of estimating $I_{param}(z_k)$: Hessian-based IFs (HIFs) and Training Trajectory-based IFs (TTIFs).

\subsection{Hessian-based Influence Functions}
\label{sec:hif bg}

Assume that $L$ is smooth and convex, and a model is well-trained on a training set using empirical risk minimization. The optimized parameters can be represented as $\theta_0=\argmin_{\theta \in \Theta} \frac{1}{n} \sum_{i=1}^n L(z_i,\theta).$
If the weight of a sample $z_k$ is marginally reduced by $\epsilon_k$, the parameters are re-optimized as $\theta_{\epsilon_k}=\argmin_{\theta \in \Theta} [\frac{1}{n} \sum_{i=1}^n L(z_i,\theta) - \epsilon_k L(z_k, \theta)].$
HIFs with respect to parameters $I^{hif}_{param}$ are derived as follows:
\begin{equation}
    \label{eq:if param}
    I_{param}^{hif}(z_k) = H_{\theta_0}^{-1}\nabla_{\theta} L(z_k,\theta_0),
\end{equation}
where $H_{\theta_0}$ is the Hessian matrix of the trained model \cite{koh_understanding_2017}. The HIFs with respect to loss $I^{hif}_{loss}$ can be formulated as:
\begin{equation}
    \label{eq:if loss}
    I_{loss}^{hif}(z_{test}, z_k) = \nabla_{\theta} L(z_{test}, \theta_0) H_{\theta_0}^{-1} \nabla_{\theta} L(z_k,\theta_0).
\end{equation}

HIFs have found applications across various domains \cite{fisher_influence_2023,lee_learning_2020,chen_fast_2023,xia_less_2024}, and a variety of methods have been proposed to mitigate computational costs \cite{agarwal_second-order_2017, martens_optimizing_2015, george_fast_2018, schioppa_scaling_2022}. 
As discussed above, to mitigate the issue that IFs always trace back to a small group of training samples, RelatIF formulates the task of tracing the most influential training samples as an optimization problem, with an additional constraint:
\begin{gather}
    \argmax_{z_k \in D} \max_{\epsilon_k} |\epsilon_k \nabla_{\theta} L(z_{test}, \theta_0) H_{\theta_0}^{-1} \nabla_{\theta} L(z_k,\theta_0)| \notag \\
    s.t. \| \epsilon_k H_{\theta_0}^{-1} \nabla_{\theta} L(z_k,\theta_0) \|_2 \leq \delta. \label{eq:constraint}
\end{gather}
By simplifying this optimization problem, $\theta$-RelatIF is obtained, which essentially performs a sample-wise normalization:
\begin{equation}
    I^{\theta-relatif}_{loss}(z_{test}, z_k) := \frac{I_{loss}(z_{test}, z_k)}{\| H_{\theta_0}^{-1}\nabla_{\theta} L(z_k, \theta_0) \|_2}.
    \label{eq:relatif}
\end{equation}
To reduce the computational cost, the Hessian matrix in IFs sometimes is assumed as identity matrix $Id$. With that, derived from $\theta$-RelatIF, we have Gradient Cosine IF $I^{cos}_{loss}(z_{test}, z_k) := \frac{\nabla_{\theta} L(z_{test}, \theta_0) \nabla_{\theta} L(z_k, \theta_0)}{\|\nabla_{\theta} L(z_{test}, \theta_0)\|_2 \|\nabla_{\theta} L(z_k, \theta_0)\|_2}.$

However, in \cite{barshan_relatif_2020},
the constraint in \eqref{eq:constraint} lacks theoretical support, and the evaluation metrics do not provide direct evidence that the traced samples are indeed the most influential for the test samples under examination.

\subsection{Training Trajectory-based Influence Functions}
As stated in Section \ref{sec:hif bg}, effectively utilizing HIFs needs assumptions which do not hold in most deep learning settings. Thus, TTIFs model the training trajectory of up/down weighting training sample to avoid these assumptions \cite{pruthi_estimating_2020,schioppa_theoretical_2023}.

Consider a deep learning model trained by mini-batch SGD optimizer. For each training step, model parameters are updated as:
\begin{equation}
    \theta_{\epsilon_k, t} = \theta_{\epsilon_k, t - 1} - \eta_{t-1} \nabla \mathcal{L}_{t-1}(\epsilon_k, \theta_{\epsilon_k, t-1}),
\end{equation}
where $\eta_t$ is learning rate (LR) at $t$ step, $\mathcal{L}_t(\epsilon_k, \theta_{\epsilon_k, t})$ denotes the mini-batch loss at $t$ step after up/down weighting the training sample $z_k$. That said, the parameters trained after T time steps are:
\begin{equation}
    \theta_{\epsilon_k, T} = \theta_{init} - \sum_{t=0}^{T-1} \eta_t \nabla \mathcal{L}_{t}(\epsilon_k, \theta_{\epsilon_k, t}),
\end{equation}
where $\theta_{init}$ is the initial parameter before training.

To capture the actual influence with SGD training, $I_{param}^{sgd}(z_k)$ is defined as the derivative of $ \theta_{\epsilon_k, T}$ with respect to $\epsilon_k$:
\begin{equation}
    \label{eq:if sgd param general}
    \begin{split}
        I_{param}^{sgd}(z_k) = & \underbrace{- \sum_{t=0}^{T-1} \eta_t \frac{d(\nabla_\theta \mathcal{L}_t(\epsilon_k, \theta_{\epsilon_k, t}))}{\epsilon_k} \bigg|_{\epsilon_k=0}}_{\uppercase\expandafter{\romannumeral1}} \\
        & \underbrace{- \sum_{t=0}^{T-1} \eta_t H_t \frac{d(\theta_{\epsilon_k, t} - \theta_{0,t})}{d\epsilon_k}\bigg|_{\epsilon_k=0}}_{\uppercase\expandafter{\romannumeral2}},
    \end{split}
\end{equation}
where $H_t$ is the Hessian matrix of $\mathcal{L}_t$. To avoid computation of $H_t$, TracIn \cite{pruthi_estimating_2020} ignores $\uppercase\expandafter{\romannumeral2}$ in \eqref{eq:if sgd param general} resulting in:
\begin{equation}
    \label{eq:tracin}
    I_{param}^{tracin}(z_k) := -\sum_{t: z_k \in batch_t} \eta_t \nabla_\theta L(z_k, \theta_{0,t}),
\end{equation}
where $batch_t$ is mini-batch used at $t$ time step.
While TracIn provides an effective way to track training trajectories, the additive assumption is questioned by \cite{guu_simfluence_2023, schioppa_theoretical_2023} and $\uppercase\expandafter{\romannumeral2}$ in \eqref{eq:if sgd param general} cannot be simply dropped. 

When $\uppercase\expandafter{\romannumeral2}$ in \eqref{eq:if sgd param general} is brought into consideration, IFs are encountered a theoretical challenge, as discussed below. Assume that both $\mathcal{L}_t$ and $\nabla_\theta \mathcal{L}_t$ with respect to $\theta$ are Lipschitz continuous with a Lipschitz constant $G$, and that $\sup_{t,\theta \in R} \|\mathcal{L}_t(\epsilon_k, \theta) - \mathcal{L}_t(0, \theta)\| \leq C\epsilon_k$, where $C$ is a constant.
By Gronwall inequality \cite{gronwall_note_1919}, Schioppa et al. \cite{schioppa_theoretical_2023} derive an upper bound for model divergence with respect to $\epsilon_k$:
\begin{equation}
    \label{eq:model divergence bound}
    \| \theta_{\epsilon_k, T} - \theta_{0, T} \| \leq C\epsilon_k \sum_{s<T} \eta_s (1+exp(2G\sum_{s<T}\eta_s)).
\end{equation}
To be Taylor expanded, IFs make critical assumptions that $ \| \theta_{\epsilon_k, T} - \theta_{0, T} \|$ is $O(\epsilon_k)$ and $\| \frac{d (\theta_{\epsilon_k, T} - \theta_{0, T})}{d\epsilon_k} \|$ is $O(1)$. However, \eqref{eq:model divergence bound} shows that the assumption of $ \| \theta_{\epsilon_k, T} - \theta_{0, T} \|$ being $O(\epsilon_k)$ may not hold. This leads to the conclusion that IFs may only be able to accurately estimate influences within a very limited time step.

While \eqref{eq:model divergence bound} seems to pose a theoretical challenge for IFs, this study derives a less conservative bound, and proves that the effectiveness of IFs is determined by the norm of gradient. Furthermore, this paper not only shows the necessity of introducing delta $\delta$ in \eqref{eq:constraint}, but also mitigate the crisis raised by \eqref{eq:model divergence bound} which indicates that the influence of a sample may not be estimated by a function of $\epsilon_k$.

\section{Methodology}
This section begins with formalizing the task and the metric used for evaluating IFs. The conditions to make IFs effective are further derived. Finally, the general form of Adjusted IFs (AIFs) is introduced to satisfy the conditions we propose. 

\subsection{Problem Statement}
\label{sec:problem definition}

Let $z_i \in D$ consists of $m$ labels $z_i=\{z_{i,1}, ... z_{i, m}\}.$ The loss function for each label $z_{i,j}$ is represented by $l$, and the overall loss $L(z_i, \theta)$ is given by the average of the individual losses, i.e., $L(z_i, \theta) = \frac{1}{m} \sum_{j=1}^m l(z_{i,j}, \theta)$.

Assume that $z_{target} \in D$ encapsulates unique knowledge that directly or indirectly lead to model output $z_{test}$ (e.g., privacy leakage) during inference stage. The objective of this work is to utilize IFs to identify $\widehat{z_{target_t}}$ such that:
\begin{equation}
    \widehat{z_{target_t}} = \argmax_{z_k \in D} I_{loss}(z_{test}, z_k).
\end{equation}
A successful trace is one where the traced target $\widehat{z_{target}}$ matches the ground truth $z_{target_t}$; otherwise, it is considered a failed trace. Accordingly, the tracing accuracy is defined as:
\begin{equation}
    \label{eq:tracing accuracy}
    Accuracy_{tracing} = \frac{N_{success}}{N_{tracing}},
\end{equation}
where $N_{success}$ is the count of successful traces and $N_{tracing}$ represents the total number of test cases during the inference stage. 

\subsection{Conditions for Using IFs in Deep Learning}
\label{sec:if conds}

Assume models are trained by a mini-batch SGD optimizer and a Learning Rate (LR) scheduler. After a sample is slightly down weighted, the change of model parameters at each training step is:
\begin{equation}
    \label{eq:param update}
    \begin{split}
        &\theta_{\epsilon_{k,j}, t}^{down} - \theta_{\epsilon_{k,j}, t-1}^{down} = - \eta_{t-1} \nabla \mathcal{L}_{t}(\epsilon_{k,j}, \theta_{\epsilon_{k,j}, t}) \\
        & = - \eta_{t-1} \sum_{i=1}^{n} B_{z_i,t-1} \nabla_\theta L(z_i, \theta_{\epsilon_{k,j}, t-1}^{down}) \\
        & + \eta_{t-1} \epsilon_{k,j} B_{z_k,t-1} \nabla_\theta l(z_{k,j}, \theta_{\epsilon_{k,j}, t-1}^{down}),
    \end{split}
\end{equation}
where $B_{z_i,t} = \left\{ \begin{aligned}
    & 1 & \text{if } z_i \in batch_t \\
    & 0 & \text{otherwise}
\end{aligned}  \right.$.

\begin{lemma}
    \label{lemma:if closed form}
    Assume that model is trained with mini-batch SGD, where $l$ is $C^2(\theta) \cap C(\epsilon_{k,j})$, and parameters satisfy \eqref{eq:param update}.

    The influence of down-weighting $z_{kj}$ on parameters is given by:
    \begin{equation}
        \label{eq:if param sgd down}
        \begin{split}
            & I_{param}^{sgd,down}(z_{k,j}) = \\
            & \sum_{t=0}^{T-2} \left(\prod_{a = t+1}^{T-1} (Id - \eta_a H_{0, a})\right) \eta_t B_{z_k, t} \nabla_\theta l(z_{k,j}, \theta_{0, t})\\
            & + \eta_{T-1} B_{z_k, T-1} \nabla_\theta l(z_{k,j}, \theta_{0, T-1}) \quad \forall T \geq 2,
        \end{split}
    \end{equation}
    where $H_{0,t}=\sum_{i=1}^n B_{z_i,t} \frac{\partial^2 L}{\partial \theta^2}$ and $\theta_{0,t}$ are the Hessian matrix and model parameters at $t$ step without altering the weight of $z_{k,j}$ respectively.
    The influence of up weighting $z_{k,j}$ on parameters is $I_{param}^{sgd,up}(z_{k,j}) =  - I_{param}^{sgd,down}(z_{k,j})$.
\end{lemma}
Proof of Lemma \ref{lemma:if closed form} can be found in Appendix \ref{sec:proof of lemma 1}.

As discussed in Section \ref{sec:related work}, the fundamental assumption of IFs is $\lim_{T\to\infty} I_{param}^{sgd} (z_{k,j})$ is $O(1)$. The following theorem demonstrates a sufficient condition for the existence of IFs in most deep learning settings.

\begin{theorem}
    \label{theorem:if cond}
    Supposing assumptions in Lemma \ref{lemma:if closed form} holds, if $\eta_t$ is monotonically decreasing for all but finitely many time steps, and there exists some time step $t_b \neq \infty$ such that $\eta_t \|H_{0,t}\|_2 < 1, \forall t \geq t_b$, then,

    \begin{equation}
        \label{eq:upper boud if i sgd}
        \begin{aligned}
            & \lim_{T\to\infty} \| I_{param}^{sgd}(z_{k,j}) \|_2 \\
            \leq & \bigg\| \sum_{t=0}^{t_b - 2} \left( \prod_{a=t+1}^{t_b-1} (Id-\eta_a H_{0,a}) \right) \eta_t B_{z_{k}, t} \nabla_{\theta}l(z_{k,j}, \theta_{0,t}) \\
            & + \eta_{t_b - 1} B_{z_k, t_b-1} \nabla_{\theta}l(z_{k,j}, \theta_{0,t_b -1}) \bigg\|_2 \\
            & + \lim_{T\to\infty} \sum_{t=t_b}^{T-1} \eta_t B_{z_k, t} \|\nabla_{\theta}l(z_{k,j}, \theta_{0,t}) \|_2.
        \end{aligned}
    \end{equation}
    
    The convergence of $\|\lim_{T \to \infty} I_{param}^{sgd}(z_{k,j})\|_2$ is determined by
    \begin{equation}
        r = \lim_{t_e \to \infty} \frac{\|\nabla_{\theta} l(z_{k,j},\theta_{0,t_f})\|_2}{\|\nabla_{\theta} l(z_{k,j},\theta_{0,t_e})\|_2},
    \end{equation}
    where $t_e < t_f$ and they are consecutive time steps such that $B_{z_k, t_e} \neq 0$ and $B_{z_k, t_f} \neq 0$. 
    
    Then, if $r < 1$, meaning that the gradient norm of $z_{k,j}$ is constantly decreasing, $\|\lim_{T \to \infty} I_{param}^{sgd}(z_{k,j})\|_2$ is convergent, i.e., $I_{param}^{sgd}(z_{k,j})$ exists.
\end{theorem}
Proof of Theorem \ref{theorem:if cond} can be found in Appendix \ref{sec:proof of theorem 1}.

Notably, the definition of LR scheduler in Theorem \ref{theorem:if cond} aligns with default HuggingFace \cite{wolf_transformers_2020} and most LMs LR schedulers with or without warmup stage, such as Linear, Constant and Exponential LR schedulers.
Theorem \ref{theorem:if cond} leads to a strong conclusion that the influence of $z_{k,j}$ may not exist if the gradient norm of $z_{k,j}$ is not decreasing or approaching to 0.

Additionally, even if $I^{sgd}_{param}(z_{k,j})$ converges, the estimated LOOR parameters $\theta_{\epsilon_{k,j}, T}^{if} = \theta_{0,T} + \epsilon_{k,j} I_{param}^{sgd}(z_{k,j})$ may not be accurate.
Removing a token from a dataset typically causes minor changes in parameters, which implies that $\|\theta_{\epsilon_{k,j}, T}\|_2$ should be small. If the gradient norm of a token is large, the RHS of \eqref{eq:upper boud if i sgd} will also be large, probably resulting in inaccurate estimation of $\theta_{\epsilon_{k,j}, T}$. Therefore, we infer that the accurate estimation should be obtained by computing $I^{sgd}_{param}(z_{k,j})$ for tokens with small gradient norm.

Next, we use Lemma \ref{lemma:if closed form} to identify the conditions under which $I_{param}^{sgd}(z_{k,j}) = I_{param}^{hif}(z_{k,j})$, and show that the estimation errors between current IFs and $I_{param}^{sgd}(z_{k,j})$ can be amplified by the gradient norm of $z_{k,j}$.
\begin{corollary}
    \label{theorem:hif unified}
    Under assumptions of Lemma \ref{lemma:if closed form}, further assume that $T \to \infty$, $\eta_t$ converges to a constant $\eta_c$, $B_{z_i, t} \equiv 1$, $\sum_{i=1}^n \nabla_{\theta} L(z_i, \theta_{0,t_{c}}) = 0$ at $t_{c}$, $\|H_{0,t_c}\|_2 < \frac{1}{\eta_c}$, and $H_{0,t_c}$ is non-singular. Then,
    \begin{equation}
        \begin{aligned}
            \lim_{T \to \infty} I^{sgd,down}_{param}(z_{k,j})
            =  H_{t_c}^{-1} \nabla_{\theta} l(z_{k,j}, \theta_{0,t_c})
        \end{aligned}
    \end{equation}
    If these assumptions do not hold, the errors introduced by the assumptions can be amplified by the gradient norm. For example, if $B_{z_i, t} \not\equiv 1$,
    \begin{equation}
        \begin{aligned}
            & \|\lim_{T \to \infty} I_{param}^{sgd}(z_{k,j}) - I_{param}^{hif}(z_{k,j})\|_2 \\
            \leq & \lim_{T \to \infty} \sum_{t=t_c}^{T-2} \bigg\| \prod_{a=t+1}^{T-1} (Id - \eta_c H_{0,t_c}) \bigg\|_2 (B_{z_k, t} - 1) \eta_c  \\
            & \|\nabla_{\theta} l(z_{k,j}, \theta_{0,t_c})\|_2 \\
            & + (B_{z_k, T-1} - 1)\eta_{c} \|\nabla_\theta l(z_{k,j}, \theta_{0, t_c})\|_2.
        \end{aligned}
    \end{equation}
\end{corollary}
Proof of Corollary \ref{theorem:hif unified} can be found in Appendix \ref{sec:proof of corollary 1}.

\begin{corollary}
    \label{theorem:ttif errors}
    Under assumptions of Lemma \ref{lemma:if closed form},
    \begin{equation}
        \label{eq:ttif error bound}
        \begin{aligned}
            &\|I_{param}^{sgd}(z_{k,j}) - I_{param}^{ttif}(z_{k,j})\|_2 \\
            \leq & \sum_{t=0}^{T-2} \bigg\| \left(\prod_{a=t+1}^{T-1}(Id-\eta_a H_{0,a})\right) - Id \bigg\|_2 \eta_t B_{z_k, t} \\
            & \|\nabla_\theta l(z_{k,j}, \theta_{0, t})\|_2.
        \end{aligned}
    \end{equation}
\end{corollary} 

As demonstrated by Corollary \ref{theorem:hif unified} and \ref{theorem:ttif errors}, the upper bound of errors for HIFs and TTIFs can increase as the gradient norm becomes larger. 
More importantly, we conducted the experiment on MNIST dataset used in \cite{koh_understanding_2017}, which verifies the consistency between $I_{loss}^{hif}$ and LOOR. We also trained a Logistic Regression Model on the same dataset with $batch\_size=50$ and $T=5000$, but changing the optimizer to mini-batch SGD. For 100 training samples, we computed HIF and TTIF as well as performed LOOR. The results on the discrepancies among the approximated $\theta_{\frac{1}{n}, k}^{if}$ with IFs and the accurate $\theta_{\frac{1}{n}, k}^{loor}$ with respect to the gradient norm of the samples are shown in Figure \ref{fig:theory plot}. Note that $\epsilon_k = \frac{1}{n}$ here, which means removing the corresponding sample from the dataset. It is clear that the estimation errors of IFs grow linearly with respect to gradient norms, aligning with the theoretical analysis above.
\begin{figure}
    \centering
    \includegraphics[width=\linewidth]{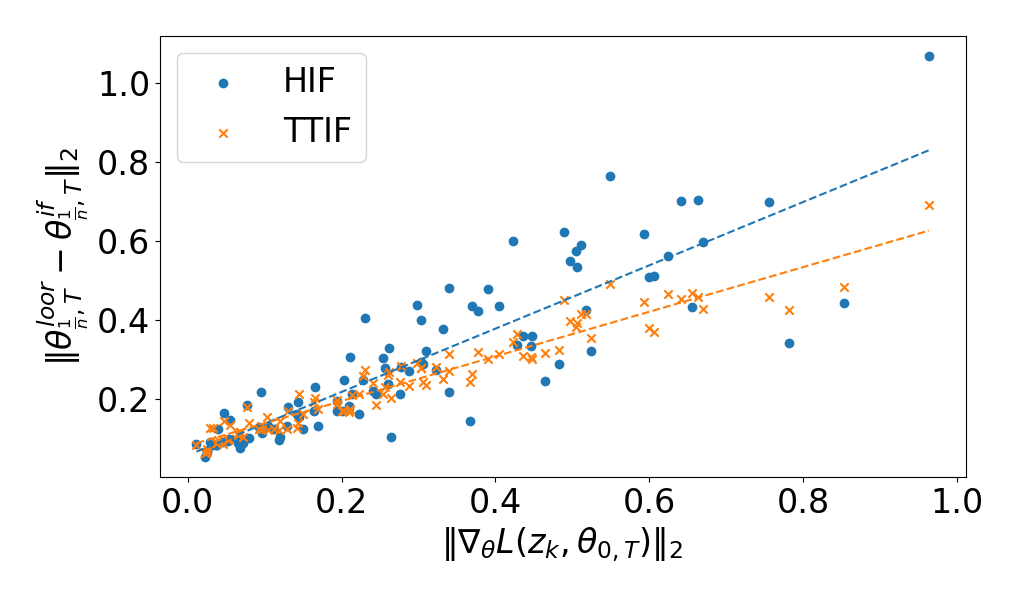}
    \caption{
        Comparison of Parameter Estimation Errors and Gradient Norms.
        We trained a Logistic Regression model using an SGD optimizer on the 10-class MNIST dataset. For HIF, we calculate the perturbed inverse of the Hessian matrix, and batch information is also brought into consideration when using TTIF.
    }
    \label{fig:theory plot}
\end{figure}

In our experiments on privacy tracing, we observed that current IFs always traced back to training samples containing tokens with large gradient norms. However, by performing LOOR, the actual influence on model parameters can be very small, sometimes even zero (examples can be found in Appendix \ref{sec:large norm small influence}). When the influences of tokens with large gradient norms are overestimated, the well-estimated influence of $z_{target}$ can be overshadowed. Therefore, in the next section, we propose Adjusted Influence Functions to reduce the weights of influence scores calculated from large gradient norms.

\subsection{Adjusted Influence Functions}
\label{sec:adjusted if}

Based on \eqref{eq:tracin}, we adjust TTIFs by introducing a function $A$ of the gradient norm, which yields:
\begin{equation}
    \label{eq:attif}
    \begin{split}
        &I_{loss}^{attif}(z_{test}, z_k) := \\
        &\nabla_\theta L(z_{test}, \theta_{0,T}) A(\sum_{j=1}^{m} \sum_{t=0}^{T-1} B_{k,t} \eta_t A(\nabla_\theta l(z_{k,j}, \theta_{0,T}))),
    \end{split}
\end{equation}
where $A(v) = W(\|v\|_2)v$, and $W: R \rightarrow R$ is monotonically decreasing function. $W$ serves an index of estimated influences, assigning larger weights to smaller gradient norms. $A$ is applied twice to further penalize the sample with large norm of gradient. Similarly, based on \eqref{eq:if loss}, HIFs can be adjusted as:
\begin{equation}
    \label{eq:ahif}
    \begin{split}
        &I_{loss}^{ahif}(z_{test}, z_k) := \\
        &\nabla_\theta L(z_{test}, \theta_{0,T}) A(\sum_{j=1}^{m} A(I_{param}^{hif}(z_{k,j}))).
    \end{split}
\end{equation}

To reduce computational cost, by only considering last time step, we propose HAIF to trace the most influential training samples with the following form:
\begin{equation}
    \label{eq:haif}
    \begin{split}
        &I_{loss}^{haif}(z_{test}, z_k) := \\
        &\nabla_\theta L(z_{test}, \theta_{0,T}) A(\sum_{j=1}^{m} A(\nabla_\theta l(z_{k,j}, \theta_{0,T}))).
    \end{split}
\end{equation}
In our experiments, HAIF utilizes $A(v) = \frac{v}{\|v\|_2}$. HAIF-T is defined as $I_{loss}^{haif-t}(z_{test}, z_k) := \nabla_\theta L(z_{test}, \theta_{0,T}) \sum_{j=1}^{m} A(\nabla_\theta l(z_{k,j}, \theta_{0,T}))$ which only adjusts gradient token-wise serving as an ablation study.
In comparison with $\theta$-RelatIF \eqref{eq:relatif}, which performs sample-wise gradient normalization, HAIF adjusts the weight of each token based on its gradient norm. If tokens with large gradient norms are not adjusted, their gradients may dominate the overall sentence gradient, thereby affecting the tracing performance.

We detail the algorithm for implementing HAIF in Appendix \ref{sec:hiaf algorithm}, which is more computationally efficient than current IFs and can accurately trace the most influential samples.

\section{Experiments}
\label{sec:exp}
This section mainly endeavors to address the following questions:
\begin{enumerate}
    \item How well can LMs memorize and extract privacy information, and subsequently utilize the memorized privacy information for reasoning?
    \item Under LM settings, can IFs accurately trace back to the $z_{target}$ according to the privacy information leaked from LMs? Moreover, does HAIF outperforms all SOTA IFs in terms of tracing accuracy.
    \item If LMs employ memorized knowledge for reasoning (specifically, when the predicted contents do not encompass any text identical to the corresponding pre-training data), can HAIF maintains higher tracing accuracy than SOTA IFs?
\end{enumerate}

\begin{table*}[hb!t]
    \centering
    \captionsetup{labelformat=empty}
    \caption{Table 2: Tracing accuracy of GPT2 using PII-E dataset. Since GPT2-102M has limited number of correct predictions, its tracing accuracy may lack of statistical significance. NAs are caused by OOM. The numbers in the table are represented in percentage form calculated by \eqref{eq:tracing accuracy}. }
    \label{tab:tracing accuracy simple gpt2}
    \resizebox{\textwidth}{!}{%
    \begin{tabular}{l|lllll|lllll|lllll|lllll}
    \toprule
    Num   Params                                            & \multicolumn{5}{c|}{102M}        & \multicolumn{5}{c|}{325M}                                                          & \multicolumn{5}{c|}{737M}                                                           & \multicolumn{5}{c}{1510M}                                                         \\ \hline
    Tracing Types                                            & DOB           & Email         & Phone      & Addr & Avg       & DOB            & Email          & Phone      & Addr           & Avg        & DOB            & Email          & Phone       & Addr           & Avg        & DOB            & Email          & Phone      & Addr           & Avg        \\ \hline \hline
    LiSSA                                                   & 100.00          & 0.00          & 0.00          & NA   & 28.57         & 0.00           & 0.00           & 0.00          & 0.00           & 0.00           & NA             & NA             & NA             & NA             & NA             & NA             & NA             & NA            & NA             & NA             \\
    EK-FAC                                                  & 100.00          & 33.00          & 50.00          & NA   & 57.14         & 16.30          & 0.00           & 33.33         & 3.03           & 13.41          & 12.03          & 0.00           & 0.00           & 7.58           & 9.56           & 1.00           & 6.90           & 0.00          & 10.77          & 9.26           \\
    RelatIF & 100.00 & 33.33 & 50.00 & NA & 57.14 & 28.15 & 0.00 & 33.33 & 9.10 & 23.46 & 20.89 & 4.55 & 0.00 & 13.64 & 17.13 & 25.63 & 17.24 & 12.50 & 16.92 & 21.85\\ 
    GradientProduct                                         & 100.00          & 0.00          & 0.00          & NA   & 28.57         & 0.00           & 0.00           & 0.00          & 0.00           & 0.00           & 0.18           & 4.55           & 0.00           & 3.03           & 2.39           & 7.50           & 3.45           & 0.00          & 3.08           & 5.56           \\
    TracInCp                                             & 100.00          & 0.00          & 0.00          & NA   & 28.57         & 0.74           & 0.00           & 0.00          & 51.52          & 10.06          & 0.18           & 4.55           & 0.00              & 3.03           & 2.39           & NA             & NA             & NA            & NA             & NA             \\
    GradientCosine                                     & 100.00          & 0.00          & 0.00          & NA   & 28.57         & 8.15           & 0.00           & 0.00          & 3.03           & 6.70           & 12.03          & 9.09           & 0.00           & 7.58           & 10.36          & 15.00          & 10.34          & 6.25          & 7.69           & 12.22          \\ \hline
    {\color[HTML]{C00000} \textbf{HAIF-T{[}Ours{]}}}     & 100.00          & 100.00          & 100.00          & NA   & 100.00          & 65.19          & 75.00          & 1.00          & 69.70          & 67.04          & 87.34          & 86.36          & 80.00          & 78.79          & 84.86          & 96.25          & 89.66          & 100.00          & 89.23          & 94.07          \\
    {\color[HTML]{C00000} \textbf{HAIF{[}Ours{]}}}         & \textbf{100.00} & \textbf{100.00} & \textbf{100.00} & NA   & \textbf{100.00} & \textbf{67.41} & \textbf{75.00} & \textbf{100.00} & \textbf{81.82} & \textbf{70.95} & \textbf{89.87} & \textbf{86.36} & \textbf{80.00} & \textbf{83.33} & \textbf{87.65} & \textbf{98.13} & \textbf{89.66} & \textbf{100.00} & \textbf{90.77} & \textbf{95.56} \\
    \bottomrule
    \end{tabular}%
    }
\end{table*}

\begin{table}[htbp]
    \centering
    \captionsetup{labelformat=empty}
    \caption{Table 3: Tracing accuracy of QWen1.5 using PII-E dataset.}
    \label{tab:tracing accuracy simple qwen}
    \resizebox{\linewidth}{!}{%
    \begin{tabular}{l|lllll|lllll}
        \toprule
        Num Params                                          & \multicolumn{5}{c|}{464M}                                                           & \multicolumn{5}{c}{1840M}                                                                                                  \\ \hline
        Tracing Types                                       & DOB            & Email          & Phone          & Addr           & Avg            & DOB                    & Email                  & Phone                  & Addr                   & Avg                    \\ \hline \hline
        LiSSA                                               & NA             & NA             & NA             & NA             & NA             & NA                     & NA                     & NA                     & NA                     & NA                     \\
        EK-FAC                                              & 37.79          & 51.79          & 44.87          & 50.43          & 45.28          & 29.63                  & 22.68                  & 23.42                  & 25.43                  & 25.89                  \\
        RelatIF & 42.44 & 54.46 & 44.87 & 57.39 & 49.27& NA & NA & NA & NA & NA \\
        GradientProduct                                     & 0.00           & 0.00           & 0.00           & 0.00           & 0.00           & 1.06                   & 0.00                   & 3.60                   & 1.69                   & 1.50                   \\
        TracInCp                                         & 0.00           & 0.00           & 0.00           & 0.00           & 0.00           & NA                     & NA                     & NA                     & NA                     & NA                     \\
        GradientCosine                                 & 1.74           & 1.78           & 1.28           & 5.22           & 2.51           & 1.59                   & 1.74                   & 5.41                   & 3.39                   & 2.81                   \\ \hline
        {\color[HTML]{C00000} \textbf{HAIF-T{[}Ours{]}}} & 66.28          & \textbf{64.29} & 71.79          & 66.96          & 66.88          & 59.26                  & 73.91                  & 66.67                  & 77.97                  & 68.10                  \\
        {\color[HTML]{C00000} \textbf{HAIF{[}Ours{]}}}     & \textbf{69.76}          & 63.39          & \textbf{74.36} & \textbf{74.78} & \textbf{70.23} & \textbf{75.66}         & \textbf{83.48}         & \textbf{72.07}         & \textbf{85.59}         & \textbf{78.80}         \\
        \bottomrule
        \end{tabular}%
    }
\end{table}

\subsection{Datasets}
\label{sec:datasets}

In order to answer the previous questions, we begin with synthesizing two datasets to have groundtruth for tracing: PII Extraction (PII-E) dataset and PII Complex Reasoning (PII-CR) Dataset. Dataset examples can be found in Appendix \ref{sec:dataset example}.
Please note that, all PIIs (e.g. name, date of birth and phone number) in this work are synthesized with Faker \footnote{https://github.com/joke2k/faker}. No real individual privacy is compromised. 

\subsubsection{PII-E}

To emulate the LLM training process, we adopt a proxy similar to \cite{allen-zhu_physics_2023} for PII-E synthesis. PII-E includes pretraining and instruction data. We generate virtual data subjects with unique names and four attributes: DOB, Email, Phone, and Address. QWen1.5-14B \cite{bai_qwen_2023} transforms this into biographies as pretraining data.
We create instruction data using the template: \textit{Question: What's the \{Attribute\} of \{Name\}? Answer: \{Attribute Value\}}. Training data combines pretraining data with 60\% of instruction data. The remaining 40\% is for testing.
Model predictions for test questions are ${z_{test}}$s. Pretraining data are ${z_i}$s for tracing. Each $z_{test}$ has a corresponding $z_{target}$ from the same virtual subject's pretraining data.
PII-E validates the PII extraction ability of LMs and the tracing ability of IFs, especially when responses include pretraining text.

\subsubsection{PII-CR}

PII-CR challenges models to perform one-step reasoning based on memorized information. It consists of pretraining and instruction data. We create mappings between festivals and dates, and landmarks and provinces. Pretraining data are generated using the template: \textit{\{Name\} is born on \{Festival Name\}. Adjacent to \{Name\}'s place is \{Landmark Building\}.}
Instruction data asks for DOB (festival date) and Address (landmark province).
PII-CR tests the PII reasoning ability of LMs and the tracing ability of IFs, when responses do not include any same text from pretraining text.

\subsection{Privacy Learning Abilities of LM}

We use classical GPT2 series and QWen1.5 series (one of the SOTA open-source LLMs \cite{huang_c-eval_2023}), and use PII-E and PII-CR datasets to fine-tune models with various series and scales. 
Generally, increased parameters and advanced architectures enhance PII extraction, with all PII types except DOB posing extraction challenges due to their length and uniqueness variations.
A similar pattern is observed in models trained on PII-CR dataset, but the need for reasoning alongside memorization reduces prediction accuracy compared to the PII-E dataset. More details can be found in Appendix \ref{sec:privacy learning}.

\subsection{Dataset Validation}
\label{sec:loo}

Let us revisit a key statement from Sections \ref{sec:problem definition} and \ref{sec:datasets}, which considers the groundtruth $z_{target}$ for a prediction $z_{test}$ is the corresponding pretraining data from the same virtual data subject.

Previous researchers \cite{koh_understanding_2017} regard LOOR as the gold standard . However, due to its sensitivity to training hyperparameters \cite{k_revisiting_2021}, detection of poisoned training data is argued to be a better metric. In our context, considering the uniqueness of PIIs, pretraining data leading to PII leakage aligns with the definition of poisoned data.

We perform LOOR on PII-E and PII-CR and training details can be found in Appendix \ref{sec:loor details}.
Table 1
shows the agreement ratio (which equivalent to tracing accuracy of LOOR) of the expected target and LOOR. 

\begin{table}[ht!]
    \centering
    \captionsetup{labelformat=empty}
    \caption{Table 1: Agreement Ratio of the Expected Target and LOOR for PII-E and PII-CR Datasets}
    \label{tab:loo}
    \resizebox{0.6\linewidth}{!}{%
    \begin{tabular}{l|ll}
    \toprule
                         & PII-E  & PII-CR \\ \hline
    Correct Prediction   & 100.00 & 13.33   \\
    Incorrect Prediction & 98.00  & 13.85   \\ \bottomrule
    \end{tabular}%
    }
\end{table}

For PII-E dataset, the agreement ratio reaches 100\%. Despite the dispute over the gold standard and long tail issue, IFs should identify the most influential pretraining sample for all $z_{test}$s in PII-E dataset.
However, the agreement ratio on PII-CR dataset is low. We believe this is because predicting the correct PII requires the model to remember the knowledge in the corresponding pre-training data and reason based on this knowledge, which is influenced by other pre-training data. Both factors affect the loss of $z_{test}$.

\begin{table*}[htbp]
    \centering
    \captionsetup{labelformat=empty}
    \caption{Table 4: Tracing accuracy of QWen1.5 and GPT2 using PII-CR dataset}
    \label{tab:tracing accuracy complex}
    \resizebox{\textwidth}{!}{%
    \begin{tabular}{l|lll|lll|lll|lll|lll|lll}
        \toprule
        Model   Series                                      & \multicolumn{6}{c|}{QWen1.5}                                                                         & \multicolumn{12}{c}{GPT2}                                                                                                                                                                                 \\ \hline
        Num Params                                          & \multicolumn{3}{c|}{464M}                         & \multicolumn{3}{c|}{1840M}                        & \multicolumn{3}{c|}{102M}                         & \multicolumn{3}{c|}{325M}                         & \multicolumn{3}{c|}{737M}                         & \multicolumn{3}{c}{1510M}                        \\ \hline
        Tracing Types                                       & DOB            & Addr           & Avg            & DOB            & Addr           & Avg            & DOB            & Addr           & Avg            & DOB            & Addr           & Avg            & DOB            & Addr           & Avg            & DOB            & Addr           & Avg            \\ \hline \hline
        LiSSA                                               & NA             & NA             & NA             & NA             & NA             & NA             & 4.17           & 10.00          & 5.88           & 2.63           & 0.00           & 2.62           & NA             & NA             & NA             & NA             & NA             & NA             \\
        EK-FAC                                              & 13.18          & 14.71          & 13.71          & \textbf{11.11} & 9.09           & 10.04          & 8.33           & 0.00           & 5.88           & 5.26           & 0.00           & 4.12           & 5.06           & 5.26           & 5.10           & 16.98          & 11.76          & 15.71          \\
        RelatIF & 21.71 & 33.82 & 25.88 & NA & NA & NA & 16.67 & 20.00 & 17.65 & 10.53 & 4.76 & 9.28 & 17.72 & 21.05 & 18.36 & 16.98 & 17.65 & 17.14 \\
        GradientProduct                                     & 0.00           & 0.00           & 0.00           & 0.00           & 0.00           & 0.00           & 4.17           & 10.00          & 5.88           & 2.63           & 0.00           & 2.06           & 2.53           & 0.00           & 2.04           & 11.32          & 17.65          & 12.86          \\
        TracInCp                                         & 0.00           & 0.00           & 0.00           & 0.00           & 0.00           & 0.00           & 4.17           & 10.00          & 5.88           & 2.63           & 0.00           & 2.06           & 2.53           & 5.26           & 3.06           & 11.32          & 17.65          & 12.86          \\
        GradientCosine                                 & 2.33           & 7.35           & 4.06           & 2.56           & 2.27           & 2.41           & 8.33           & 20.00          & 11.76          & 6.58           & 4.76           & 6.19           & 8.86           & 21.05          & 11.22          & 1.89           & 17.65          & 5.71           \\ \hline
        {\color[HTML]{C00000} \textbf{HAIF-T{[}Ours{]}}} & 31.78          & 55.88          & 40.10          & 2.56           & 18.94          & 11.24          & 16.67          & 20.00          & 17.65          & 50.00          & 33.33          & 46.39          & 58.23          & 78.95          & 62.24          & 47.17          & 47.06          & 47.14          \\
        {\color[HTML]{C00000} \textbf{HAIF{[}Ours{]}}}     & \textbf{34.88} & \textbf{55.88} & \textbf{42.13} & 2.56           & \textbf{22.73}          & \textbf{13.25} & \textbf{16.67} & \textbf{20.00}          & \textbf{17.65}          & \textbf{51.32} & \textbf{33.33} & \textbf{47.42} & \textbf{60.76} & \textbf{78.95} & \textbf{64.29} & \textbf{50.94}          & \textbf{47.06}          & \textbf{50.00}          \\
        \bottomrule
        \end{tabular}%
    }
\end{table*}

\begin{figure*}[htbp]
    \centering
    \begin{subfigure}{0.33\textwidth}
        \centering
        \includegraphics[width=\linewidth]{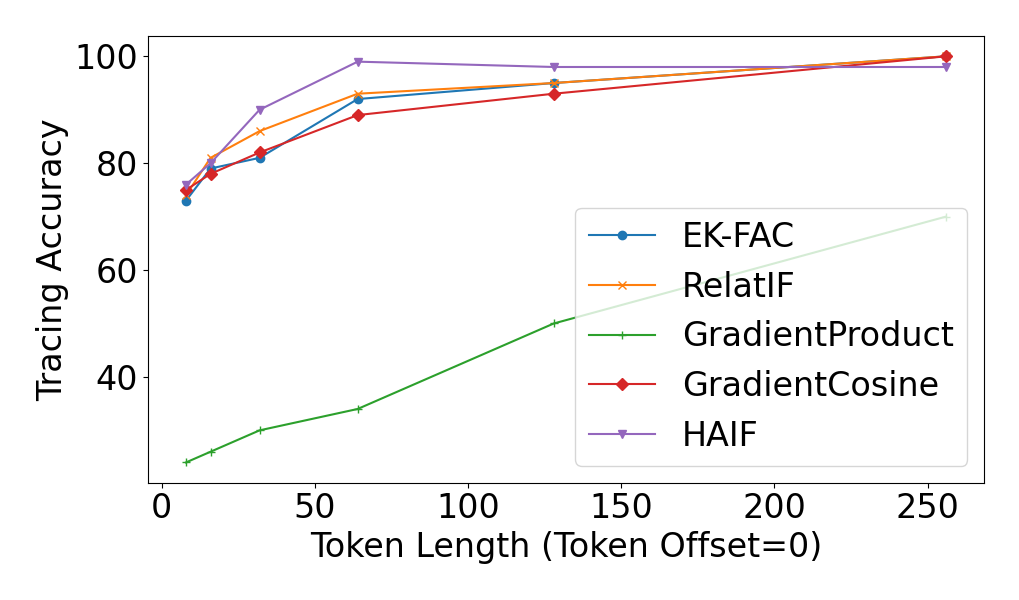}
        \caption{Tracing accuracy vs token length without token offset.}
        \label{fig:length offset=0}
    \end{subfigure}
    \hfill
    \begin{subfigure}{0.33\textwidth}
        \centering
        \includegraphics[width=\linewidth]{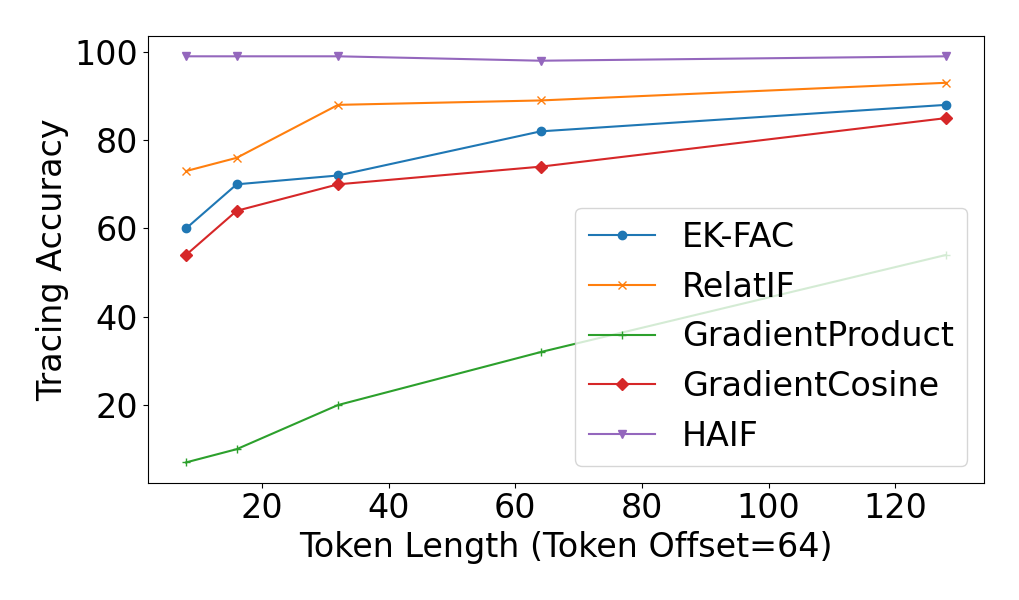}
        \caption{Tracing accuracy vs token length with fixed token offset 64.}
        \label{fig:length offset=64}
    \end{subfigure}
    \hfill
    \begin{subfigure}{0.33\textwidth}
        \centering
        \includegraphics[width=\linewidth]{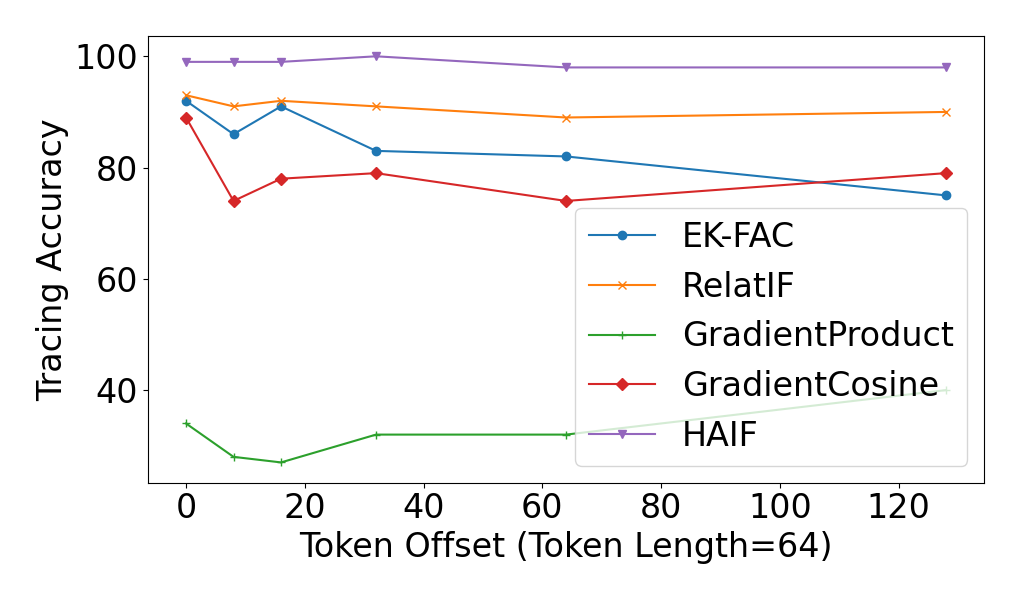}
        \caption{Tracing accuracy vs token offset with fixed token length 64.}
        \label{fig:offset length=64}
    \end{subfigure}
    \caption{Tracing accuracy on CLUECorpus2020 dataset with different token offsets and lengths.}
\end{figure*}

\subsection{Privacy Tracing Accuracy}

We compare the tracing accuracy of five SOTA influence functions (IFs) on various model series, scales, datasets, and PII types. We select five SOTA IFs discussed in Section \ref{sec:related work} as baselines, i.e. EK-FAC \cite{grosse_studying_2023}, LiSSA \cite{bae_if_2022,koh_understanding_2017}, RelatIF \cite{barshan_relatif_2020}, TracInCP \cite{pruthi_estimating_2020}, GradientProduct and GradientCosine. The detailed configurations of IFs are listed in Appendix \ref{sec:if configs}.

\subsubsection{PII Extraction Tracing Abilities}
\label{sec:pii-e tracing}
In our experiments, incorrectly predicted privacy information is not regarded as a threat; therefore, we only trace the correct predictions of each model. Please note that the maximum tracing accuracy on PII-E is 100\% regardless of the choice of gold standard.

Table 2
shows the tracing accuracy of IFs on various GPT2 scales on PII-E dataset. HAIF significantly improves tracing accuracy across all GPT2 scales and PII types. Comparing to the best baseline, HAIF-T enhances average tracing accuracy from 43.58\% to 73.71\%. Tracing accuracy of HAIF-T increases as models scale up, achieving 94.07\% on GPT2-1.5B. This performance is consistent across PII types. Additionally, HAIF boosts tracing accuracy for all model scales by 1.49\% to 3.91\%, reaching 95.56\% on GPT2-1.5B. Table 3
presents tracing accuracies for IFs on QWen1.5 models on PII-E dataset. HAIF still surpasses the best baseline, achieving tracing accuracies of 70.23\% and 78.80\% on QWen series models.

\subsubsection{PII Reasoning Tracing Accuracy}

While PII-E is a good benchmark for IFs, its outputs contain exact pretraining data content, enabling good results via text search. Hence, we use PII-CR dataset to test IFs when LM outputs result from reasoning, not identical pretraining data.

The comparison results are reported in Table 4
showing the tracing accuracies of IFs on various GPT2 and QWen1.5 scales on PII-CR dataset. Note that maximum tracing accuracy is not 100\% due to 5.5\% PII prediction accuracy from random guessing on PII-CR dataset. For example, tracing accuracy of QWen1.5-0.5B on PII-CR should be less than 88.83\%. Lower PII predictions thus yield lower maximum tracing accuracies.

Generally, increased task difficulty leads to decreased prediction and tracing accuracies compared to PII-E dataset. Despite LM outputs lacking identical pretraining data, HAIF maintains highest tracing accuracy across all models. For GPT2 series, HAIF significantly improves tracing accuracy, ranging from 32.86\% to 45.93\%, over the best baseline. For QWen1.5 series, improvements range from 3.21\% to 16.25\%.

\subsection{Tracing Accuracy on CLUECorpus2020}

To further validates the effectiveness of the proposed HAIF, we conduct experiments on the real-world pretraining dataset, CLUECorpus2020 \cite{xu_cluecorpus2020_2020}
The text is tokenized and concatenated to form the pretraining dataset, where $n=100$ and $m=256$. The test set mirrors the training set, with the addition of two hyperparameters, offset and length, to control the test loss, defined as $L(z_{test_i}, \theta) = \sum_{j=offset}^{offset+length} l(z_{i,j}, \theta)$. Here, the offset simulates the user prompt, and the length denotes the model response length. Through LOOR experiments, each $z_i$ is the most influential training sample for $z_{test_i}$.
According to previous experiments, the best baseline occurs while using QWen1.5-0.5B; thus it is employed to test various IFs under different lengths and offsets.

Fig. \ref{fig:length offset=0} illustrates that without introducing token offset, RelatIF, EK-FAC, GradientCosine, and HAIF exhibit comparable tracing accuracy as token length varies. This can be attributed to the fact that large gradient norms are mainly found in the initial tokens. Therefore, when the $offset=0$, $z_{test_i}$ and $z_i$ share identical gradients with large norms, enabling RelatIF, EK-FAC and GradientCosine to maintain high tracing accuracies.
However, as depicted in Fig. \ref{fig:length offset=64}, the introduction of an offset (which simulates user prompts) leads to a decrease in tracing accuracy for all IFs, with the exception of HAIF. The robustness of HAIF is further demonstrated in Fig. \ref{fig:offset length=64}, where HAIF maintains its performance trend, unaffected by the increase in token offset.
It is noteworthy that the tracing accuracy improves with token length, regardless of token offsets. Considering that user prompts are typically not empty and privacy information typically comprises a limited number of tokens, HAIF outperforms all SOTAs in such scenarios. 

\subsection{Conclusions}

This work implements IFs to trace back to the most influential pre-training data to address the concern in privacy leakage in LMs. 
Our analysis reveals that tokens with large gradient norms can increase errors in LOOR parameter estimation and amplify the discrepancy between existing IFs and actual influence.
Therefore, we propose a novel IF, HAIF, which reduces the weights of token with large gradient norms, providing high tracing accuracy with less computational cost than other IFs. To better validate the proposed methods, we construct two datasets, PII-E and PII-CR, where the groundtruth of privacy leakage can be located. Experiments on different model series and scales demonstrate that HAIFs significantly improve tracing accuracy in comparison with all SOTA IFs.
Moreover, on real-world pretraining data CLUECorpus2020, HAIF consistently outperforms all SOTA IFs exhibiting strong robustness across varying prompt and response lengths.

\bibliography{anonymous-submission-latex-2025.bbl}

\clearpage 
\appendix

\section{Appendix}
\subsection{Proof of Lemma \ref{lemma:if closed form}}
\label{sec:proof of lemma 1}

\begin{lemma*}
    Assume that model is trained with mini-batch SGD, where $l$ is $C^2(\theta) \cap C(\epsilon_{k,j})$, and parameters satisfy 

    \begin{equation}
        \begin{split}
            \theta_{\epsilon_{k,j}, t}^{down} - \theta_{\epsilon_{k,j}, t-1}^{down} = & - \eta_{t-1} \sum_{i=1}^{n} B_{z_i,t-1} \nabla_\theta L(z_i, \theta_{\epsilon_{k,j}, t-1}^{down}) \\
            & + \eta_{t-1} \epsilon_{k,j} B_{z_k,t-1} \nabla_\theta l(z_{k,j}, \theta_{\epsilon_{k,j}, t-1}^{down}).
        \end{split}
    \end{equation}

    The influence of down-weighting $z_{kj}$ on parameters is given by:
    \begin{equation}
        \begin{split}
            \label{eq:if sgd down appendix}
            & I_{param}^{sgd,down}(z_{k,j}) = \\
            & \sum_{t=0}^{T-2} \left( \prod_{a = t+1}^{T-1} (Id - \eta_a H_{0, a}) \right) \eta_t B_{z_k, t} \nabla_\theta l(z_{k,j}, \theta_{0, t})\\
            & + \eta_{T-1} B_{z_k, T-1} \nabla_\theta l(z_{k,j}, \theta_{0, T-1}) \quad \forall T \geq 2,
        \end{split}
    \end{equation}
    where $H_{0,t}=\sum_{i=1}^n B_{z_i,t} \frac{\partial^2 L}{\partial \theta^2}$ and $\theta_{0,t}$ are the Hessian matrix and model parameters at $t$ step without altering the weight of $z_{k,j}$ respectively.
\end{lemma*}
\begin{proof}
    \begin{subequations}
        \label{eq:param update expansion}
        \begin{equation}
            \begin{aligned}
                \theta_{\epsilon_{k,j}, 1}^{down} - \theta_{init} = & - \eta_{0} \sum_{i=1}^{n} B_{z_i,0} \nabla_\theta L(z_i, \theta_{\epsilon_{k,j}, 0}^{down}) \\
                & + \eta_{0} \epsilon_{k,j} B_{z_k,0} \nabla_\theta l(z_{k,j}, \theta_{\epsilon_{k,j}, 0}^{down}),
            \end{aligned}
        \end{equation}
        \begin{equation}
            \begin{aligned}
                \theta_{\epsilon_{k,j}, 2}^{down} - \theta_{\epsilon_{k,j}, 1}^{down} = & - \eta_{1} \sum_{i=1}^{n} B_{z_i,1} \nabla_\theta L(z_i, \theta_{\epsilon_{k,j}, 1}^{down}) \\
                & + \eta_{1} \epsilon_{k,j} B_{z_k,1} \nabla_\theta l(z_{k,j}, \theta_{\epsilon_{k,j}, 1}^{down}),
            \end{aligned}
        \end{equation}
        \begin{equation*}
            \vdots
        \end{equation*}
        \begin{equation}
            \begin{aligned}
                & \theta_{\epsilon_{k,j}, T}^{down} - \theta_{\epsilon_{k,j}, T-1}^{down} = \\
                & - \eta_{T-1} \sum_{i=1}^{n} B_{z_i,T-1} \nabla_\theta L(z_i, \theta_{\epsilon_{k,j}, T-1}^{down}) \\
                & + \eta_{T-1} \epsilon_{k,j} B_{z_k,T-1} \nabla_\theta l(z_{k,j}, \theta_{\epsilon_{k,j}, T-1}^{down}).
            \end{aligned}
        \end{equation}
    \end{subequations}
    
    By summing over equations in \eqref{eq:param update expansion}, the parameters at $T$ time step is:
    \begin{equation}
        \begin{split}
            \theta_{\epsilon_{k,j}, T}^{down} = & \theta_{init}- \sum_{t=0}^{T-1} \eta_t \sum_{i=1}^n B_{z_i,t} \nabla_\theta L(z_i,\theta_{\epsilon_{k,j},t}^{down}) \\
            & + \sum_{t=0}^{T-1} \eta_t \epsilon_{k,j} B_{z_k,t} \nabla_\theta l(z_{k,j}, \theta_{\epsilon_{k,j},t}^{down}).
        \end{split}
    \end{equation}
    By taking derivative of $\theta_{\epsilon_{k,j}, T}^{down}$ with respect to $\epsilon_{k,j}$,
    \begin{equation}
        \begin{split}
            &\frac{d \theta_{\epsilon_{k,j}, T}^{down}}{d \epsilon_{k,j}} \\= & - \sum_{t=0}^{T-1} \eta_t \sum_{i=1}^n B_{z_i,t} \frac{d(\nabla_\theta L(z_i,\theta_{\epsilon_{k,j},t}^{down}))}{d \theta_{\epsilon_{k,j},t}^{down}}\frac{\theta_{\epsilon_{k,j},t}^{down}}{d \epsilon_{k,j}} \\
            & + \sum_{t=0}^{T-1} \eta_t B_{z_k,t} \nabla_\theta l(z_{k,j}, \theta_{\epsilon_{k,j},t}^{down}) \\
            & + \sum_{t=0}^{T-1} \eta_t \epsilon_{k,j} B_{z_k,t} \frac{d(\nabla_\theta l(z_{k,j}, \theta_{\epsilon_{k,j},t}^{down}))}{d \theta_{\epsilon_{k,j},t}^{down}}\frac{\theta_{\epsilon_{k,j},t}^{down}}{d \epsilon_{k,j}} \\
            = & \sum_{t=0}^{T-1} \eta_t B_{z_k,t} \nabla_\theta l(z_{k,j}, \theta_{\epsilon_{k,j},t}^{down}) - \eta_t H_{\epsilon_{k,j},t} \frac{d \theta_{\epsilon_{k,j}, t}^{down}}{d \epsilon_{k,j}}.
        \end{split}
    \end{equation}
    Then,
    \begin{equation}
        \label{eq:if sgd down before expansion}
        \begin{split}
            & I_{param}^{sdg,down}(z_{k,j}) := \frac{d \theta_{\epsilon_{k,j}, T}^{down}}{d \epsilon_{k,j}} \bigg|_{\epsilon_{k,j}=0} \\
            & = \sum_{t=0}^{T-1} \eta_t B_{z_k,t} \nabla_\theta l(z_{k,j}, \theta_{0,t}) - \eta_t H_{0,t} \frac{d \theta_{\epsilon_{k,j}, t}^{down}}{d \epsilon_{k,j}}\bigg|_{\epsilon_{k,j}=0},
        \end{split}
    \end{equation}
    where $H_{\epsilon_{k,j},t} = \sum_{i=1}^n B_{z_i,t} \frac{\partial^2 L}{\partial \theta^2} - \epsilon_{k,j}B_{z_k, t} \frac{\partial^2 l}{\partial \theta^2}$.

    Thus, $I_{param}^{sdg,down}(z_{k,j}) = \eta_0 B_{z_k,0} \nabla_\theta l(z_{k,j}, \theta_{0,0})$, when $T=1$.

    Next, we prove \eqref{eq:if sgd down appendix} by mathematical induction.

    Base case: When $T=2$,
    \begin{equation}
        \label{eq:bs}
        \begin{split}
            & I_{param}^{sdg,down}(z_{k,j}) \\
            = & \eta_0 B_{z_k,0} \nabla_\theta l(z_{k,j}, \theta_{0,0}) + \eta_1 B_{z_k,1} \nabla_\theta l(z_{k,j}, \theta_{0,1}) \\
            & - \eta_1 H_{0,1} \eta_0 B_{z_k,0} \nabla_\theta l(z_{k,j}, \theta_{0,0}) \\
            = & (Id-\eta_1 H_{0,1})\eta_0 B_{z_k,0} \nabla_\theta l(z_{k,j}, \theta_{0,0}) \\ 
            & + \eta_1 B_{z_k,1} \nabla_\theta l(z_{k,j}, \theta_{0,1}),
        \end{split}
    \end{equation}
    which satisfies \eqref{eq:if sgd down appendix}.

    I.H.: Suppose the statement is true for any time step $T=q$, for some $q \geq 2$, namely,
    \begin{equation}
        \label{eq:ih}
        \begin{split}
            & I_{param}^{sgd,down}(z_{k,j}) = \\
            & \sum_{t=0}^{q-2} \left( \prod_{a = t+1}^{q-1} (Id - \eta_a H_{0, a}) \right) \eta_t B_{z_k, t} \nabla_\theta l(z_{k,j}, \theta_{0, t})\\
            & + \eta_{q-1} B_{z_k, q-1} \nabla_\theta l(z_{k,j}, \theta_{0, q-1}).
        \end{split}
    \end{equation}

    I.S.: For $T=q+1$, by expanding \eqref{eq:if sgd down before expansion},
    
    \begin{equation}
        \label{eq:is}
        \begin{split}
            & I_{param}^{sdg,down}(z_{k,j}) \\
            = & \quad \eta_0 B_{z_k,0} \nabla_\theta l(z_{k,j}, \theta_{0,0}) \\
            & + \eta_1 B_{z_k,1} \nabla_\theta l(z_{k,j}, \theta_{0,1}) \\
            & + \eta_2 B_{z_k,2} \nabla_\theta l(z_{k,j}, \theta_{0,2}) \\
            & + \dots \\
            & + \eta_q B_{z_k,q} \nabla_\theta l(z_{k,j}, \theta_{0,q}) \\
            & - \eta_1 H_{0,1} \frac{d \theta_{\epsilon_{k,j}, 1}^{down}}{d \epsilon_{k,j}}\bigg|_{\epsilon_{k,j}=0} \\
            & - \eta_2 H_{0,2} \frac{d \theta_{\epsilon_{k,j}, 2}^{down}}{d \epsilon_{k,j}}\bigg|_{\epsilon_{k,j}=0} \\
            & -  \dots \\
            & - \eta_q H_{0,q} \frac{d \theta_{\epsilon_{k,j}, q}^{down}}{d \epsilon_{k,j}}\bigg|_{\epsilon_{k,j}=0} \\
            = & \quad \eta_0 B_{z_k,0} \nabla_\theta l(z_{k,j}, \theta_{0,0}) \\
            & + \eta_1 B_{z_k,1} \nabla_\theta l(z_{k,j}, \theta_{0,1}) \\
            & + \eta_2 B_{z_k,2} \nabla_\theta l(z_{k,j}, \theta_{0,2}) \\
            & + \dots \\
            & + \eta_q B_{z_k,q} \nabla_\theta l(z_{k,j}, \theta_{0,q}) \\
            & - \eta_1 H_{0, 1} \eta_0 B_{z_k,0} \nabla_\theta l(z_{k,j}, \theta_{0,0}) \\
            & - \eta_2 H_{0, 2} (Id - \eta_1 H_{0,1}) \eta_0 B_{z_k,0} \nabla_\theta l(z_{k,j}, \theta_{0,0}) \\
            & - \eta_2 H_{0, 2} \eta_1 B_{z_k,1} \nabla_\theta l(z_{k,j}, \theta_{0,1}) \\
            & - \dots \\
            & - \eta_q H_{0,q} \prod_{a=1}^{q-1} (Id-\eta_a H_{0,a}) \eta_0 B_{z_k,0} \nabla_\theta l(z_{k,j}, \theta_{0,0}) \\
            & - \eta_q H_{0,q} \prod_{a=2}^{q-1} (Id-\eta_a H_{0,a}) \eta_1 B_{z_k,1} \nabla_\theta l(z_{k,j}, \theta_{0,1}) \\
            & - \eta_q H_{0,q} \prod_{a=3}^{q-1} (Id-\eta_a H_{0,a}) \eta_2 B_{z_k,2} \nabla_\theta l(z_{k,j}, \theta_{0,2}) \\
            & - \dots \\
            & - \eta_q H_{0,q} \prod_{a=q-1}^{q-1} (Id-\eta_a H_{0,a}) \eta_{q-2} B_{z_k,q-2} \\
            &\nabla_\theta l(z_{k,j}, \theta_{0,q-2}) \\
            & - \eta_q H_{0, q}\eta_{q-1} B_{z_k,q-1} \nabla_\theta l(z_{k,j}, \theta_{0,q-1}) \\
            = & \sum_{t=0}^{q-1} \left( \prod_{a=t+1}^{q} (Id-\eta_a H_{0,a}) \right) \eta_t B_{z_k, t} \nabla_\theta l(z_{k,j}, \theta_{0, t}) \\
            & + \eta_q B_{z_k, q}\nabla_\theta l(z_{k,j}, \theta_{0, q}).
        \end{split}
    \end{equation}
    Therefore, the proof is completed.
\end{proof}

\subsection{Proof of Theorem \ref{theorem:if cond}}
\label{sec:proof of theorem 1}

\begin{theorem*}
    Supposing assumptions in Lemma \ref{lemma:if closed form} holds, if $\eta_t$ is monotonically decreasing for all but finitely many time steps, and there exists $t_b \neq \infty$ such that $\eta_t \|H_{0,t}\|_2 < 1, \forall t \geq t_b$, then,

    \begin{equation}
        \begin{aligned}
            & \lim_{T\to\infty} \| I_{param}^{sgd}(z_{k,j}) \|_2 \\
            \leq & \bigg\| \sum_{t=0}^{t_b - 2} \left( \prod_{a=t+1}^{t_b-1} (Id-\eta_a H_{0,a}) \right) \eta_t B_{z_{k}, t} \nabla_{\theta}l(z_{k,j}, \theta_{0,t}) \\
            & + \eta_{t_b - 1} B_{z_k, t_b-1} \nabla_{\theta}l(z_{k,j}, \theta_{0,t_b -1}) \bigg\|_2 \\
            & + \lim_{T\to\infty} \sum_{t=t_b}^{T-1} \eta_t B_{z_k, t} \|\nabla_{\theta}l(z_{k,j}, \theta_{0,t}) \|_2.
        \end{aligned}
    \end{equation}
    
    The convergence of $\|\lim_{T \to \infty} I_{param}^{sgd}(z_{k,j})\|_2$ is determined by
    \begin{equation}
        r = \lim_{t_e \to \infty} \frac{\|\nabla_{\theta} l(z_{k,j},\theta_{0,t_f})\|_2}{\|\nabla_{\theta} l(z_{k,j},\theta_{0,t_e})\|_2},
    \end{equation}
    where $t_e < t_f$ and they are consecutive time steps such that $B_{z_k, t_e} \neq 0$ and $B_{z_k, t_f} \neq 0$. 

    Then, if $r < 1$, meaning that the gradient norm of $z_{k,j}$ is constantly decreasing, $\|\lim_{T \to \infty} I_{param}^{sgd}(z_{k,j})\|_2$ is convergent, i.e., $I_{param}^{sgd}(z_{k,j})$ exists.
    
\end{theorem*}
\begin{proof}
    \begin{equation}
        \label{eq:cond1 if sgd}
        \begin{aligned}
            & \lim_{T\to\infty} \| I_{param}^{sgd}(z_{k,j}) \|_2 \\
            \leq & \underbrace{\lim_{T\to\infty} \bigg\| \sum_{t=0}^{t_b - 1} \left( \prod_{a=t+1}^{T-1} (Id-\eta_a H_{0,a}) \right) \eta_t B_{z_{k}, t} }_{\uppercase\expandafter{\romannumeral1}} \\
            & \underbrace{\nabla_{\theta}l(z_{k,j}, \theta_{0,t}) \bigg\|_2}_{\uppercase\expandafter{\romannumeral1}} \\
            & + \underbrace{\lim_{T\to\infty} \bigg\| \sum_{t=t_b}^{T-2} \left( \prod_{a=t+1}^{T-1} (Id-\eta_a H_{0,a}) \right) \eta_t B_{z_{k}, t}}_{\uppercase\expandafter{\romannumeral2}} \\
            & \underbrace{ \nabla_{\theta}l(z_{k,j}, \theta_{0,t}) + \eta_{T-1} B_{z_k, T-1} \nabla_\theta l(z_{k,j}, \theta_{0, T-1}) \bigg\|_2}_{\uppercase\expandafter{\romannumeral2}}.
        \end{aligned}
    \end{equation}
    Given that $\eta_b \|H_{0,t}\|_2 < 1, \forall t \geq t_b$, 
    \begin{equation}
        \lim_{T\to\infty} \|\prod_{a=t+1}^{T-1} (Id-\eta_a H_{0,a}) \|_2 \leq 1, \quad \forall t \geq t_b.
    \end{equation}
    Thus,
    \begin{equation}
        \begin{aligned}
            & \| \uppercase\expandafter{\romannumeral1} \|_2 \\
            \leq & \bigg\| \sum_{t=0}^{t_b - 2} \left( \prod_{a=t+1}^{t_b-1} (Id-\eta_a H_{0,a}) \right) \eta_t B_{z_{k}, t} \nabla_{\theta}l(z_{k,j}, \theta_{0,t}) \\
            & + \eta_{t_b - 1} B_{z_k, t_b-1} \nabla_{\theta}l(z_{k,j}, \theta_{0,t_b -1}) \bigg\|_2,
        \end{aligned}
    \end{equation}
    which is convergent. We also have
    \begin{equation}
        \begin{aligned}
            \| \uppercase\expandafter{\romannumeral2} \|_2
            \leq &\bigg\| \sum_{t=t_b}^{T-1} \eta_t B_{z_k, t} \nabla_{\theta}l(z_{k,j}, \theta_{0,t}) \bigg\|_2 \\
            \leq & \sum_{t=t_b}^{T-1} \eta_t B_{z_k, t} \|\nabla_{\theta}l(z_{k,j}, \theta_{0,t}) \|_2.
        \end{aligned}
    \end{equation}
    By ratio test, the sufficient condition for upper bound of $\| \uppercase\expandafter{\romannumeral2} \|_2$ being convergent is
    \begin{equation}
        \lim_{t_e\to\infty}\frac{\eta_{t_f}\|\nabla_{\theta}l(z_{k,j}, \theta_{0,t_f}) \|_2}{\eta_{t_e}\|\nabla_{\theta}l(z_{k,j}, \theta_{0,t_e}) \|_2} < 1.
    \end{equation}
    Let 
    \begin{equation}
        r = \lim_{t_e\to\infty}\frac{\|\nabla_{\theta}l(z_{k,j}, \theta_{0,t_f}) \|_2}{\|\nabla_{\theta}l(z_{k,j}, \theta_{0,t_e}) \|_2}.
    \end{equation}
    Since $\eta_{t_e} \geq \eta_{t_f}$ when $t_e < t_f$,
    \begin{equation}
        \lim_{t_e\to\infty}\frac{\eta_{t_f}\|\nabla_{\theta}l(z_{k,j}, \theta_{0,t_f}) \|_2}{\eta_{t_e}\|\nabla_{\theta}l(z_{k,j}, \theta_{0,t_e}) \|_2} \leq r.
    \end{equation}
    If $r < 1$, $\| I_{param}^{sgd}(z_{k,j}) \|_2$ is convergent. If $r > 1$ or $\lim_{t\to\infty} \|\nabla_{\theta}l(z_{k,j}, \theta_{0,t}) \|_2 \neq 0$, the upper bound of $\| I_{param}^{sgd}(z_{k,j}) \|_2$ can be divergent.
\end{proof}

\subsection{Other Sufficient Conditions for $I_{param}^{sgd}(z_{k,j})$ Being Convergent}
The following Corollaries demonstrate other sufficient conditions for $I_{param}^{sgd}(z_{k,j})$ being convergent. However, as all the Hessian matrices $H_{0,t}, \forall t \geq t_b$ are required to be non-singular, they may not be useful for general deep learning models.

\begin{corollary*}
    Under assumption of Lemma \ref{lemma:if closed form}, if $\eta_t$ converges to a sufficiently small constant $\eta_b$ such that $\eta_b \|H_{0,t}\|_2 < 1, \forall t \geq t_b$, and $H_{0,t}, \forall t \geq t_b$ are non-singular, then $\lim_{T \to \infty} I_{param}^{sgd}(z_{k,j})$ is convergent. Specifically, when $B_{z_k, t} \equiv 1$, $\|\lim_{T \to \infty} I_{param}^{sgd}(z_{k,j})\|_2 \leq \|M^{-1}v\|$, where $M = \argmin_{\{H_{0,t}; t \geq t_b\}} \sigma_{min}(H_{0,t})$, $v = \argmax_{\{\nabla_{\theta} l(z_{k,j},\theta_{0,t}); t \geq t_b\}} \|M\nabla_{\theta} l(z_{k,j},\theta_{0,t})\|_2$, and $\sigma_{min}$ denotes the minimum singular value.
\end{corollary*}
\begin{proof}
    If $\eta_t$ converges to a sufficiently small constant $\eta_b \neq 0$ such that $\eta_b \|H_{0,t}\|_2 < 1, \forall t \geq t_b$, and $H_{0,t}, \forall t \geq t_b$ are non-singular, then 
    \begin{equation}
        \begin{aligned}
            \lim_{T\to\infty} \|\prod_{a=t_b+1}^{T-1} (Id-\eta_a H_{0,a}) \|_2 = 0,
        \end{aligned}
    \end{equation}
    and thereby $\uppercase\expandafter{\romannumeral1} = 0$.
    \begin{equation}
        \begin{aligned}
            & \lim_{T\to\infty}\|I_{param}^{sgd}(z_{k,j})\|_2 \\
            \leq & \lim_{T\to\infty} \left\| \sum_{t=t_b}^{T-1} (Id - \eta_b M)^{T-t-1} \eta_b B_{z_{k}, t} v \right\|_2.
        \end{aligned}
    \end{equation}
    Since $0 < \|H_{0,t}\|_2 < 1, \forall t \geq t_b$, $\lim_{T\to\infty}\|I_{param}^{sgd}(z_{k,j})\|_2$ is convergent. 
    Specifically, if $B_{z_{k}, t} \equiv 1$, by Neumann series, $\lim_{T\to\infty}\|I_{param}^{sgd}(z_{k,j})\|_2 \leq \lim_{T\to\infty} \| \sum_{t=t_b}^{T-1} (Id - \eta_b M)^{T-t-1} \eta_b v \|_2 = \|M^{-1}v\|_2$.
\end{proof}

Empirical studies demonstrate that most of eigenvalues of Hessian matrices are clustered near 0 in deep learning models. Hence,$H_{0,t}, \forall t \geq t_b$ being non-singular is a strong assumption and may not be satisfied for most real-world models.

\begin{corollary*}
    Under assumption of Lemma \ref{lemma:if closed form}, if $\eta_t$ converges to 0 in finite time step, then $\lim_{T \to \infty} I_{param}^{sgd}(z_{k,j})$ is a finite series and is therefore convergent.
\end{corollary*}
The proof of this corollary is straightforward, as each element in the sequence does not go to infinity.
However, in most LMs training settings, $\eta_t$ is not 0 until $t$ reaches $T$.

\subsection{Proof of Corollary \ref{theorem:hif unified}}
\label{sec:proof of corollary 1}

\begin{corollary*}
    Under assumptions of Lemma \ref{lemma:if closed form}, further assume that $T \to \infty$, $\eta_t$ converges to a constant $\eta_c$, $B_{z_i, t} \equiv 1$, $\sum_{i=1}^n \nabla_{\theta} L(z_i, \theta_{0,t_{c}}) = 0$ at $t_{c}$, $\|H_{0,t_c}\|_2 < \frac{1}{\eta_c}$, and $H_{0,t_c}$ is non-singular. Then,
    \begin{equation}
        \begin{aligned}
            \lim_{T \to \infty} I^{sgd,down}_{param}(z_{k,j})
            =  H_{t_c}^{-1} \nabla_{\theta} l(z_{k,j}, \theta_{0,t_c})
        \end{aligned}
    \end{equation}
    If these assumptions do not hold, the errors introduced by the assumptions can be amplified by the gradient norm. For example, if $B_{z_i, t} \not\equiv 1$,
    \begin{equation}
        \label{eq:hif mini batch error}
        \begin{aligned}
            & \|\lim_{T \to \infty} I_{param}^{sgd}(z_{k,j}) - I_{param}^{hif}(z_{k,j})\|_2 \\
            \leq & \lim_{T \to \infty} \sum_{t=t_c}^{T-2} \bigg\| \prod_{a=t+1}^{T-1} (Id - \eta_c H_{0,t_c}) \bigg\|_2 (B_{z_k, t} - 1) \eta_c  \\
            & \|\nabla_{\theta} l(z_{k,j}, \theta_{0,t_c})\|_2 \\
            & + (B_{z_k, T-1} - 1)\eta_{c} \|\nabla_\theta l(z_{k,j}, \theta_{0, t_c})\|_2.
        \end{aligned}
    \end{equation}
\end{corollary*}

\begin{proof}
    \begin{equation}
        \label{eq:proof hif}
        \begin{aligned}
            & \lim_{T \to \infty} I^{sgd,down}_{param}(z_{k,j}) \\
            = & \underbrace{\lim_{T \to \infty}  \sum_{t=0}^{t_c - 1} \left( \prod_{a=t+1}^{T-1} (I-\eta_c H_{0,a}) \right) \eta_t \nabla_{\theta} l(z_{k,j}, \theta_{0,t})}_{\uppercase\expandafter{\romannumeral1}} \\
            & + \underbrace{\lim_{T \to \infty} \sum_{t=t_c}^{T-2} \left( \prod_{a=t+1}^{T-1} (1 - \eta_c H_{0,t_c}) \right) \eta_c \nabla_{\theta} l(z_{k,j}, \theta_{0,t_c})}_{\uppercase\expandafter{\romannumeral2}}\\
            & \underbrace{+ \eta_{c}\nabla_\theta l(z_{k,j}, \theta_{0, t_c})}_{\uppercase\expandafter{\romannumeral2}}
        \end{aligned}
    \end{equation}
    Given that $\|H_{0,t_c}\| \leq \frac{1}{\eta_c}$, $\lim_{T \to \infty}\prod_{a=t_c}^{T-1} (I-\eta_c H_{0, t_c})=0$; consequently, $\uppercase\expandafter{\romannumeral1}$ in \eqref{eq:proof hif} is 0.
    \begin{equation}
        \uppercase\expandafter{\romannumeral2} = \lim_{T \to \infty} \sum_{t=t_c}^{T-1} (I-\eta_c H_{0,t_c})^{T-t_c-1} \eta_c \nabla_\theta l(z_{k,j}, \theta_{0, t_c})
    \end{equation}
    Given $\|H_{0,t_c}\| \leq \frac{1}{\eta_c}$, by Neumann series, $\lim_{T \to \infty} I^{sgd,down}_{param}(z_{k,j}) = H_{t_c}^{-1} \nabla_{\theta} l(z_{k,j}, \theta_{0,t_c})$.

    If $B_{z_i, t} \not\equiv 1$, \eqref{eq:hif mini batch error} can be directly proven by Cauchy-Schwarz and Triangle inequalities.
\end{proof}

\subsection{Algorithm of HAIF}
\label{sec:hiaf algorithm}

The algorithm implementing HAIF is summarized in Algorithm \ref{alg:haif}. 
Weights for adjusting outputs are initially computed and cached in function output\_adjustment (line 12-19). Notably, function output\_adjustment only needs to be run once for all test samples. Subsequently, the gradient of the test sample is determined (line 3). Following this, an iteration over the training samples is performed to compute the dot product between the gradient of the test sample and each individual training sample (line 4-11). 
In comparison to the current IFs, HAIF is computational efficient and able to accurately trace the most influential samples.

\begin{algorithm}[ht!]
    \caption{HAIF Algorithm}
    \label{alg:haif}
    \begin{algorithmic}[1]
        \Require $z_{test}, \{z_k\}_{k=1}^{n}, \theta_T, l, L$
        \Ensure $\{s_k\}_{k=1}^n$
        \State $\{s_k\}_{k=1}^n \gets \{0\}_1^n$
        \State $\{w_{k,j}\} \gets output\_adjustment(\theta_T,\{z_k\}_{k=1}^{n}, l)$
        \State $g_{t} \gets grad(L(z_{test}, \theta^c)) $
        \For{$k \gets 1$ to $n$}
            \State $L_k \gets 0$
            \For{$j \gets 1$ to $m$}
                \State $L_k = L_k + w_{k,j}l(z_{k,j},\theta^c)$
            \EndFor
            \State $g_k \gets grad(L_k)$
            \State $s_k = s_k + A(g_t^T)$
        \EndFor
        \Function{output\_adjustment}{$\theta_T,\{z_k\}_{k=1}^{n}, l$}
            \For{$k \gets 1$ to $n$}
                \For{$j \gets 1$ to $m$}
                    \State $w_{k,j} \gets W(\|grad(l(z_{k,j},\theta_T))\|)$
                \EndFor
            \EndFor
            \State \Return $\{w_{k,j}\}$
        \EndFunction
    \end{algorithmic}
\end{algorithm}

\subsection{Dataset Examples for PII-E and PII-CR}
\label{sec:dataset example}
Here, we provide more details about PII-E and PII-CR.
Table \ref{tab:pii-e example} and \ref{tab:pii-cr example} illustrate two real examples in datasets. Name and all attributes are generated by Faker \footnote{https://github.com/joke2k/faker}.

Table \ref{tab:pii-e example} illustrates the attributes, pretraining data, and instruction data of a data subject, Liao Tingting, in the PII-E dataset. For example, models are trained to predict `1912/01/01' for the DOB of Liao Tingting. Furthermore, IFs are expected to trace this prediction back to the pretraining data of Liao Tingting.

PII-CR is more challenging, as shown in Table \ref{tab:pii-cr example}. When asked for the DOB of Liao Tingting, models are expected to answer `01/01'. IFs are expected to trace this prediction to the pretraining data, where `01/01' does not exist, while the festival New Year is contained. Models should understand the link among Liao Tingting, `01/01', and New Year. IFs are expected to trace across such reasoning.

\begin{table}[ht!]
    \centering
    \caption{An example of PII-E dataset}
    \label{tab:pii-e example}
    \resizebox{\linewidth}{!}{%
    \begin{tabular}{lllll}
        \toprule
        \multicolumn{5}{c}{Virtual   Data Subject}                                                                                                                                       \\ \hline
        \multicolumn{1}{l|}{Name}          & \multicolumn{1}{l|}{DOB}        & \multicolumn{1}{l|}{Email}          & \multicolumn{1}{l|}{Phone}       & Address                          \\ \hline
        \multicolumn{1}{l|}{Liao Tingting} & \multicolumn{1}{l|}{1912/01/01} & \multicolumn{1}{l|}{yqiu@yahoo.com} & \multicolumn{1}{l|}{18859382421} & Block N, Qiqihar Road... \\ \hline \hline
        \multicolumn{5}{c}{Pretraining   (Biology)}                                                                                                                                      \\ \hline
        \multicolumn{5}{l}{Liao Tingting was   born on Jan. 1, 1912, at Block N, Qiqihar Road... She is a hard ...}                                                           \\ \hline \hline
        \multicolumn{5}{c}{Instructions}                                                                                                                                                  \\ \hline
        \multicolumn{3}{l|}{Q: What's the DOB of   Liao Tingting?}                                                 & \multicolumn{2}{l}{A: 1912/01/01}                                   \\ \hline
        \multicolumn{3}{l|}{Q: What's the Email   of Liao Tingting?}                                               & \multicolumn{2}{l}{A: yqiu@yahoo.com}                               \\ \hline
        \multicolumn{3}{l|}{Q: What's the Phone   of Liao Tingting?}                                               & \multicolumn{2}{l}{A: 18859382421}                                  \\ \hline
        \multicolumn{3}{l|}{Q: What's the   Address of Liao Tingting?}                                             & \multicolumn{2}{l}{A: Block N, Qiqihar Road, Xincheng…}             \\ \bottomrule
        \end{tabular}%
    }
\end{table}

\begin{table}
    \centering
    \caption{An example of PII-CR dataset}
    \label{tab:pii-cr example}
    \resizebox{\linewidth}{!}{%
    \begin{tabular}{lll}
        \toprule
        \multicolumn{3}{c}{Virtual   Data Subject}                                                                                                                         \\ \hline
        \multicolumn{1}{l|}{Name}                                       & \multicolumn{1}{l|}{DOB}                               & Address                                 \\ \hline
        \multicolumn{1}{l|}{Liao Tingting}                              & \multicolumn{1}{l|}{01/01}                             & Shanghai                                \\ \hline \hline
        \multicolumn{3}{c}{Pretraining   (Biology)}                                                                                                                        \\ \hline
        \multicolumn{3}{l}{\begin{tabular}[c]{@{}l@{}}Liao Tingting is   born on New Year. Adjacent to Liao\\ Tingting's place is Oriental Pearl TV   Tower.\end{tabular}} \\ \hline \hline
        \multicolumn{3}{c}{Instructions}                                                                                                                                   \\ \hline
        \multicolumn{2}{l|}{Q: What's the DOB of   Liao Tingting?}                                                               & A: 01/01                                \\ \hline
        \multicolumn{2}{l|}{Q: What's the   Address of Liao Tingting?}                                                           & A: Shanghai                            
        \\ \bottomrule
    \end{tabular}%
    }
\end{table}

\subsection{Detailed Privacy Learning Abilities of LM}
\label{sec:privacy learning}

\begin{figure*}[htbp]
    \centering

    \begin{subfigure}{0.49\textwidth}
        \centering
        \includegraphics[width=\linewidth]{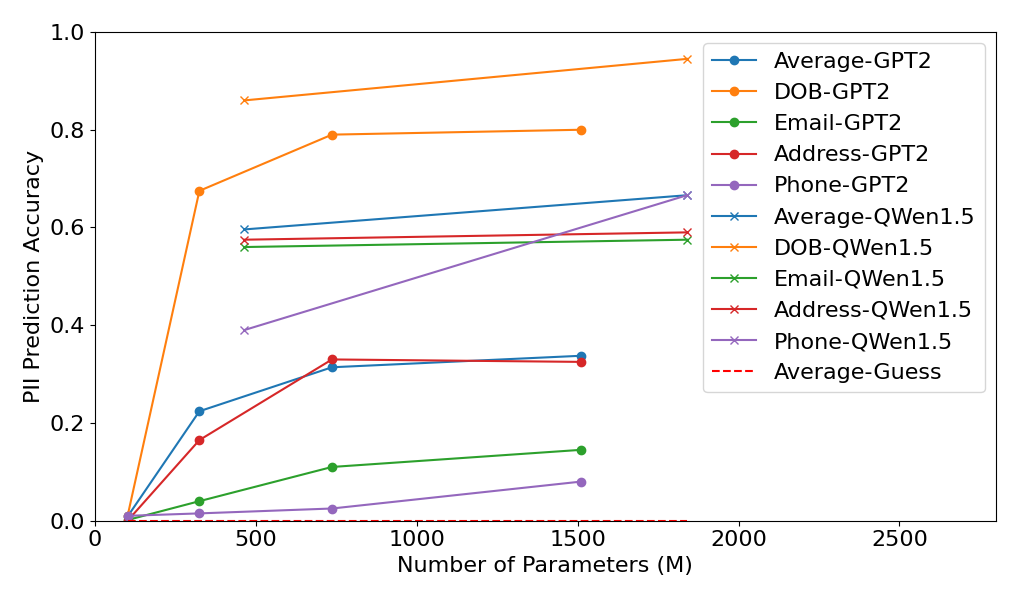}
        \caption{Results on PII-E Dataset}
        \label{fig:pii-e}
    \end{subfigure}
    \hfill
    \begin{subfigure}{0.49\textwidth}
        \centering
        \includegraphics[width=\linewidth]{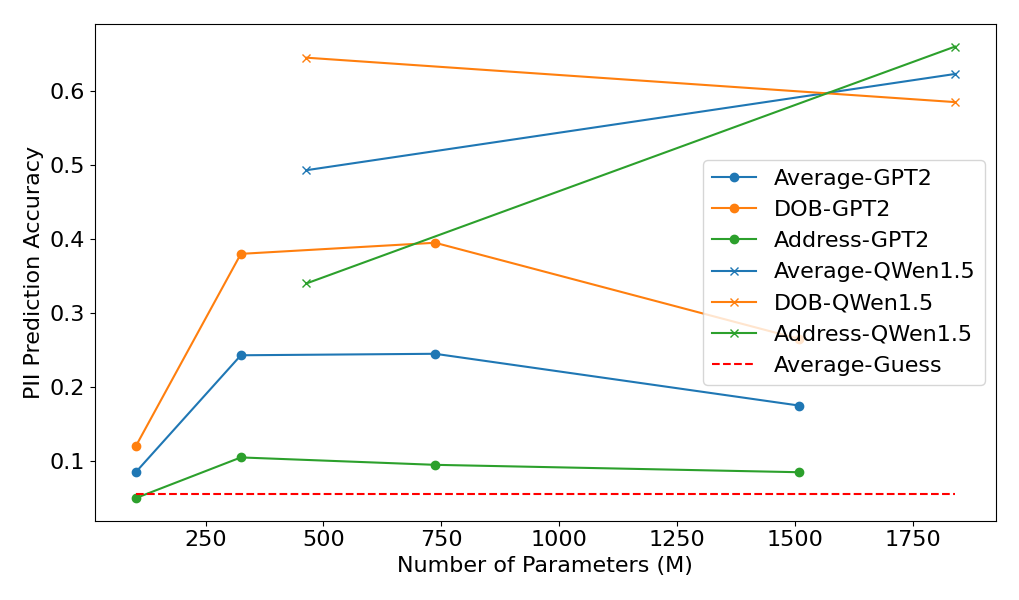}
        \caption{Results on PII-CR Dataset}
        \label{fig:pii-cr}
    \end{subfigure}
    \caption{PII Prediction Accuracy of GPT2 and QWen1.5. The red dashed line represents the random guess performance of the QWen1.5-0.5B model, trained exclusively with instruction data}
\end{figure*}

In this work, we use classical GPT2 series and QWen1.5 series (one of the SOTA open-source LLMs \cite{huang_c-eval_2023}) as base models. We use PII-E and PII-CR datasets to fine-tune various model series and scales. Since model training aims to construct as many traceable samples as possible for evaluating tracing performance, each model are trained with 50 epochs while saving the model with highest PII prediction accuracy. With this in mind, no over-fitting assumption is involved when evaluating tracing accuracy. All models are trained with AdamW optimizer with default settings and linear scheduler. As regular training procedure leads to OOM for QWen1.5-1.8B and GPT2-xLarge, we train them with DeepSpeed-ZeRO2 \cite{rajbhandari_zero_2020}.

The PII prediction accuracy of GPT2 and QWen1.5 on PII-E dataset is illustrated in Fig. \ref{fig:pii-e}. 
In general, as the number of parameters increases, so does the ability to memorize and extract PII. More advanced and contemporary model architectures also enhance PII extraction capabilities. Considering the variations in length and uniqueness among different types of PII, all PII types except for DOB present more extraction challenges. 
As shown in Fig. \ref{fig:pii-cr}, similar pattern is also observed in PII-CR dataset. However, the dataset requires models to reason based on memorized PII information in addition to memorization, the PII prediction accuracy significantly decreases compared to PII-E dataset. 
We further trained the QWen1.5-0.5B model solely using the instruction data from the two datasets. The performance of model was equivalent to a random guess. In the case of PII-E, the model was unable to infer the linkage between names and privacy attributes due to the uniqueness of PII. However, for PII reasoning datasets PII-CR, even without pretraining data, the models could potentially make a correct prediction if they learned the answer format. This random guess behavior greatly influenced the LOOR results, as demonstrated in Section \ref{sec:loo}.

Given that even the smallest QWen1.5 model performs better (with overwhelming superiority) in the PII prediction accuracy across all PII types compared to the largest GPT2 model, we can conclude that as model architectures become more advanced and their fundamental capabilities are enhanced, models can extract personal information from more obscure expressions, thereby increasing the risk of privacy leakage.

\subsection{The Actual Influences of Tokens}
\label{sec:large norm small influence}

\begin{figure}
    \centering
    \includegraphics[trim={0 0.7cm 0 0},clip,width=\linewidth]{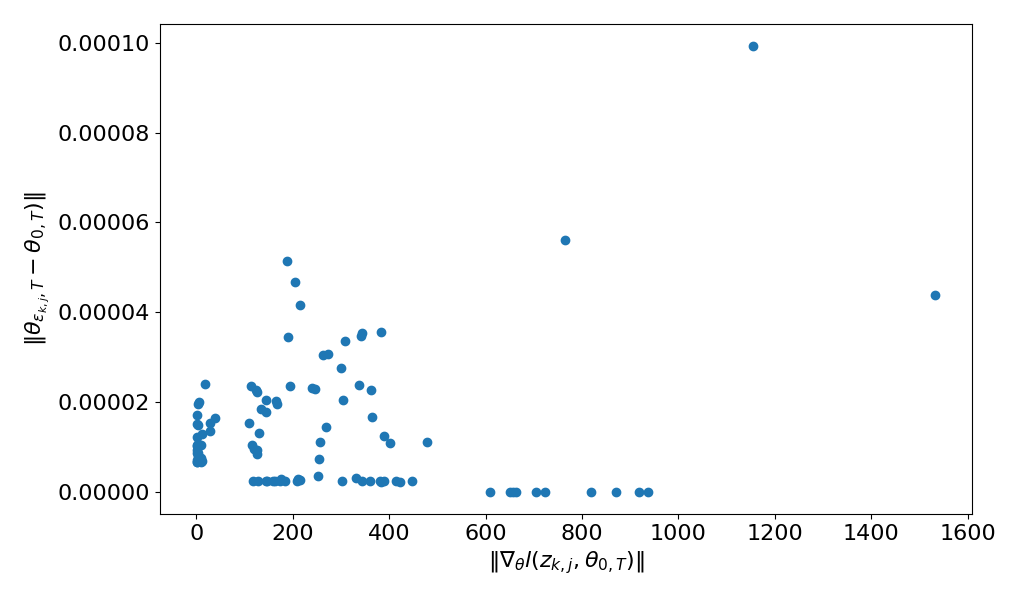}
    \caption{Comparison of Token Influences and Token Gradient Norms}
    \label{fig:large norm small influence}
\end{figure}

In this section, we conduct experiments to verify that the influences with large gradient norm on model parameters can be small or even zero. Instead of leave-one-sample-out retraining, we perform leave-one-token-out experiments on the PII-E dataset using the QWen1.5-0.5B model with the SGD optimizer. Since performing LOOR for all tokens is intractable, we sample 10 tokens from each of 10 samples and plot the relationship between the parameter change of each token and its gradient norm. As shown in Figure \ref{fig:large norm small influence}, tokens typically have minimal influence on parameters (less than $10^{-4}$), and tokens with large gradient norms can have small or zero influences on model parameters.

\subsection{LOOR Details on PII-E and PII-CR}
\label{sec:loor details}

LOOR for all pretraining samples is time-consuming. Hence, a subset of 100 pretraining samples and their instruction data are selected. To balance time and PII prediction accuracy, we choose QWen1.5-0.5B. We sequentially remove each pretraining data and compute the actual loss change for each $z_{test}$: $\Delta L(z_{test}, \theta^{LOOR}_{\frac{1}{n}, z_k})$. If the expected $target_{z_{test}}$ equals $\argmax_{k}\Delta L(z_{test}, \theta^{LOOR}_{\frac{1}{n}, z_k})$, LOOR is considered to agree with the expectation. This allows us to calculate the agreement ratio of the expected target and LOOR. Note that PII-E and PII-CR datasets share the same retraining proxy.

\subsection{Configurations for IFs}
\label{sec:if configs}
This section details the configurations of various IFs in experiments.
For LiSSA, we adhere to the settings used in \cite{bae_if_2022}. Gauss Newton Hessian (GNH) matrix is enabled, and the depth of LiSSA is set to match the length of the dataset. However, even with GNH, calculating Jacobian of model output still leads to Out-Of-Memory (OOM) for large models.
We employ the default settings for EK-FAC. Given that the EK-FAC algorithm supports limited types of deep learning layers, we transform the 1D-CNN layer into an equivalent Linear layer as guided in \cite{grosse_studying_2023}. We disregard layers in QWen1.5 series that lead to errors, such as Rotary Embedding and Layer Norm and no layer is disregarded for GPT2 series. 
$l$-RelatIF ($I^{l-relatif}_{loss}(z_{test}, z_k) = \frac{I^{hif}_{loss}(z_{test}, z_k)}{\sqrt{I^{hif}_{loss}(z_k, z_k)}}$) is utilized during experiments which, according to \cite{barshan_relatif_2020}, provides a better estimation than $\theta$-RelatIF.
In the case of TracInCp, we utilize three checkpoints saved during the training process. While additional checkpoints may potentially enhance tracing accuracy, it would also result in a longer time than EK-FAC, thereby negating the speed advantage of first-order IFs.

\end{document}